\documentclass[12pt]{article}

\usepackage{fullpage}
\usepackage{amsmath}
\usepackage{amsfonts}
\usepackage{standalone}
\usepackage{booktabs}
\usepackage{mathtools}

\usepackage{tikz}
\usepackage{comment}
\usetikzlibrary{shapes,arrows,chains}
\usepackage{pgfplots}
\pgfplotsset{width=7cm,compat=1.8}
\usepackage{pgfplotstable}
\renewcommand*{\familydefault}{\sfdefault}
\usepackage{sfmath}
\usepackage{graphicx}
\usepackage{subcaption}
\usepackage{mwe}
\usepackage{amsthm}
\usepackage{multirow}
\usepackage{floatrow}
\usepackage{authblk}

\usepackage{lineno,hyperref}
\usepackage{xcolor}
\usepackage{color}

\usepackage[ruled, vlined]{algorithm2e}

\def\flatnetwork{\mathcal{G}}

\def\flatnodes{\mathcal{N}}
\def\flatarcs{\mathcal{A}}

\def\suppliers{\mathcal{S}}
\def\supplier{s}
\def\warehouses{\mathcal{W}}
\def\warehouse{w}

\def\customers{\mathcal{C}}
\def\customer{c}
\def\timehorizon{\mathcal{T}}

\def\time{t}
\def\products{\mathcal{P}}
\def\product{p}

\def\expnetwork{\mathcal{G}_\timehorizon}

\def\expnodes{\mathcal{N}_\timehorizon}
\def\exptarcs{\mathcal{A}_\timehorizon}
\def\tarc{((i,t),(j,t'))}
\def\harc{((i,t),(i,t+1))}
\def\expharcs{\mathcal{H}_\timehorizon}
\def\flowstvariables{x^{ptt'}_{ij}}
\def\flowshvariables{x^{ptt+1}_{ii}}
\def\trucksvariables{y^{tt'}_{ij}}

\def\flowshvariablesaggregated{x^{\superproduct tt+1}_{ii}}
\def\flowstkvariablesaggregated{x^{\superproduct_k tt'}_{ij}}
\def\flowshkvariablesaggregated{x^{\superproduct_k tt+1}_{ii}}
\def\flow{x}
\def\truck{y}
\def\z{z}

\def\fcost{f_{ij}}
\def\ltcost{c_{ij}}
\def\lhcost{c_{ii}}
\def\transittime{t_{ij}}

\def\demand{d^p_{ct}}
\def\d{d}
\def\D{D}
\def\trucksalloc{\bar{y}}
\def\trucksvariablesalloc{\bar{y}^{tt'}_{ij}}
\def\superproducts{\Xi}
\def\superproduct{\chi}

\def\truckcapacity{\hat u}

\def\extremerays{\Omega}
\def\extremeray{\rho}
\def\extremepoints{\Gamma}
\def\extremepoint{\pi}
\def\K{\mathcal{K}}

\def\paths{\Lambda}
\def\path{\lambda}
\def\pathflow{\gamma}
\def\matching{m}

\def\timelimitphaseone{t_1^{max}}
\def\timelimitphasetwo{t_2^{max}}
\def\timelimitlb{t^{bounds}}
\def\imprlb{impr^{bounds}}
\def\minmastersols{msols_{max}}

\def\productpercentage{\phi}
\def\families{\mathcal{F}}
\def\proximity{d}
\def\warehousecapacity{lim^{\warehouses}_i}

\newfloatcommand{capbtabbox}{table}[][\FBwidth]

\definecolor{linkcol}{rgb}{0,0,0.4} 
\definecolor{citecol}{rgb}{0.5,0,0} 
\definecolor{linkcol}{rgb}{0,0,0} 
\definecolor{citecol}{rgb}{0,0,0}
\definecolor{myOrange}{rgb}{1,0.5,0}
\definecolor{myDarkOrange}{rgb}{.698,.349,0}
\definecolor{mySoftOrange}{rgb}{.996,.875,.749}
\definecolor{myVerySoftOrange}{rgb}{.996,.749,.502}	
\definecolor{myOrange}{rgb}{1,0.5,0}
\definecolor{colorRoot1}{rgb}{.184,.062,.8}
\definecolor{colorRoot2}{rgb}{.0,.408,.627}
\definecolor{colorRoot3}{rgb}{.55,.455,.828}	
\definecolor{colorRoot4}{rgb}{1,.847,.341}	
\definecolor{colorTime}{rgb}{.757,.263,.294}

\newtheorem{theorem}{Theorem}[section]

\providecommand{\keywords}[1]
{
  \small	
  \textbf{\textit{Keywords---}} #1
}

\title{Meta Partial Benders Decomposition for the Logistics Service Network Design Problem}
\author[1]{Simon Belieres}
\author[2]{Mike Hewitt}
\author[3]{Nicolas Jozefowiez}
\author[4]{Fr\'ed\'eric Semet}
\affil[1]{CNRS, LAAS, 7 Avenue du Colonel Roche, 31077 Toulouse Cedex 4, France}
\affil[2]{Quinlan School of Business, Loyola University, 16 E. Pearson Ave., IL, Chicago 60611, USA }
\affil[3]{LCOMS EA 7306, Universit\'e de Lorraine, Metz 57000, France}
\affil[4]{Universit\'e de Lille, CNRS, Centrale Lille, Inria, UMR 9189 - CRIStAL, F-59000 Lille, France}
\date{}                     
\begin{document}

\maketitle

\vspace{-1cm}

\begin{abstract}
Supply chain transportation operations often account for a large proportion of product total cost to market. Such operations can be optimized by solving the Logistics Service Network Design Problem (LSNDP), wherein a logistics service provider seeks to cost-effectively source and fulfill customer demands of products within a multi-echelon distribution network. However, many industrial settings yield instances of the LSNDP that are too large to be solved in reasonable run-times by off-the-shelf optimization solvers. We introduce an exact Benders decomposition algorithm based on partial decompositions that strengthen the master problem with information derived from aggregating subproblem data. More specifically, the proposed Meta Partial Benders Decomposition intelligently switches from one master problem to another by changing both the amount of subproblem information to include in the master as well as how it is aggregated. Through an extensive computational study, we show that the approach outperforms existing benchmark methods and we demonstrate the benefits of dynamically refining the master problem in the course of a partial Benders decomposition-based scheme.
\end{abstract}


\section{Introduction}
\label{sec:intro}
\noindent

Supply chains are complex networks of partner stakeholders that create products and distribute them to a consumer market. These stakeholders can be decomposed into echelons, wherein each echelon group stakeholders that serve similar roles. A major share of the supply chain operating costs is due to freight transportation within the multi-echelon distribution network that links stakeholders. As a result, cost-effective transportation operations are essential to ensure profitability. An increasingly common trend to reduce distribution costs is to outsource supply chain management activities to third-party logistics (3PL) service providers \cite{Marasco:2008} that specialize in the integration of warehousing and transportation services. In this paper, we present a new algorithmic strategy for solving a transportation problem encountered by a 3PL partner in the management of a restaurant supply chain. 

In the problem studied, restaurants place orders of products for the coming weeks and the 3PL must design a transportation plan to fulfill these orders. As the products ordered by the restaurants are generic, i.e. they can be shipped by different suppliers of the supply chain, designing the transportation plan requires sourcing the products ordered by the restaurants and routing these products through a multi-echelon distribution network. This planning process can be assisted through the solution of the Logistics Service Network Design Problem (LSNDP). The LSNDP is a tactical transportation problem with few dedicated studies. Dufour et al. \cite{Dufour:2018} present a case study of the United Nations Humanitarian Response Depot, a humanitarian logistics service provider. They study a supply chain that aims to distribute relief goods and equipments in East Africa and they present a logistics service network design model for optimizing the transportation operations. They solve instances based on different network configurations and they demonstrate that substantial savings can be achieved by integrating a new distribution center. Belieres et al. \cite{Belieres:2020} describe the LSNDP studied in this article. They investigate how the distribution strategy used by the 3PL can negatively impact overall logistics costs and they assess the savings that can be generated by extending this distribution strategy.

Belieres et al. \cite{Belieres:2019} propose an effective Benders decomposition-based algorithm to solve this problem. Their strategy includes valid inequalities, heuristic solutions, as well as an advanced decomposition strategy inspired by the recently-proposed \textit{Partial Benders Decomposition} (PBD). Introduced in the context of solving two-stage stochastic programs, the PBD retains some subproblem structure into the master problem to reduce the number of iterations and accelerate the algorithm convergence. In a similar fashion, Belieres et al. \cite{Belieres:2019} propose to reinforce the master problem by including variables and constraints that model the need to satisfy customer demands of a ``super-product'' that aggregates all the products to be routed. The authors demonstrate that strengthening the master problem with this aggregated information enables an enhanced Benders decomposition-based algorithm to produce provably high-quality solutions in much less time than a general-purpose mixed-integer programming solver. In this paper, we present multiple PBD-related enhancements to the algorithm presented in \cite{Belieres:2019}, yielding an algorithm that exhibits much better computational performance.

First, whereas Belieres et al. \cite{Belieres:2019} consider a single super-product, we consider multiple. Specifically, we propose determining a set of super-products and then adding variables and constraints to the master problem related to routing each super-product. To determine a set of super-products of a given cardinality, we propose partitioning the set of products and associating a super-product with each set of the partition. In addition, we characterize settings wherein the set of products can be partitioned to form super-products that yield a master problem provably equivalent to the original problem. Namely, a master problem that enables a Benders decomposition-based algorithm to converge in a single iteration. We use that characterization as the basis of a strategy for partitioning the set of products into a given number of subsets, and resulting super-products, in more general settings.

One would expect a trade-off associated with using multiple super-products to formulate the master problem. On the one hand, a greater number of super-products should yield a stronger master problem and enable the algorithm to converge in fewer iterations. For example, considering more super-products will increase the likelihood that the resulting master problem is provably equivalent to the original problem. On the other hand, a greater number of super-products will likely yield a master problem that is harder to solve, and thus the algorithm will spend more time executing each iteration. However, it is not yet known how to accurate estimate the impact of either on the execution of the resulting Benders decomposition-based algorithm. 

Thus, we propose a new algorithmic strategy that exploits this feature by intelligently varying the amount of subproblem information used to formulate the master problem. We refer to this strategy as Meta Partial Benders Decomposition (Meta-PBD). Meta-PBD operates in two phases, with the first aiming to explore different areas within the feasible region of the original problem and to quickly determine high-quality solutions. To do so, the number of super-products used to formulate the master problem changes dynamically while respecting a threshold value to ensure computational tractability. The second phase aims to close the optimality gap. To do so, it drastically increase the amount of subproblem information used to formulate the master problem, such that the latter becomes provably equivalent to the original problem. While this master problem is computationally challenging, high-quality solutions generated through the first phase enable many nodes of the resulting branch-and-bound search tree to be pruned, accelerating convergence.

Our contribution is twofold. First, we provide theorical results on how to aggregate subproblem information to determine a strong partial Benders decomposition for the LSNDP. While the results presented are specific to the LSNDP, they can easily be extended to other multi-product network design problems. We also investigate how the level of aggregation of the subproblem information impacts the resulting partial Benders decomposition. To the best of our knowledge, this has not been studied in the literature. Second, we introduce a new exact algorithmic strategy for solving the LSNDP. Through an extensive series of experiments performed on random instances, we show that Meta-PBD determines provably high-quality solutions in limited running times and strictly outperforms the algorithm proposed by Belieres et al. \cite{Belieres:2019}, including solving more instances to optimality. We also demonstrate the effectiveness of our approach on a set of industrial instances \cite{Belieres:2020}.

The remainder of the paper is organized as follows. In Section 2, we review the relevant literature. In Section 3, we present a description and formulation of the Logistics Service Network Design Problem. In Section 4, we present a Benders decomposition-based algorithm for this problem as well as different super-product-based master problems. In Section 5, we present the Meta-PBD algorithm. In Section 6, we assess the performance of Meta-PBD with a series of computational experiments. In Section 7, we draw conclusions and discuss future work.

\section{Literature review}
\label{sec:lit_review}
\noindent
In this section, we first review the literature relevant to the Logistics Service Network Design Problem (LSNDP). Then, we review the literature documenting advanced decomposition strategies for enhancing Benders decomposition-based algorithms.

Since the LSNDP \cite{Dufour:2018,Belieres:2019} has received little attention to date, we position it regarding existing network design problems. The LSNDP is similar to the Service Network Design Problem (SNDP) \cite{Crainic:2000, Wieberneit:2008} in that both problems focus on transportation planning decisions within a facility network. However, the LSNDP considers supply chains wherein products are made by multiple suppliers and thus fulfilling customer orders involves a sourcing decision. This is in contrast to most variants of the SNDP studied in the literature, wherein the origin of goods to be transported is presumed to be given. Besides, the SNDP makes no presumptions regarding the structure of the network, while the LSNDP presumes a multi-echelon structure. On the other hand, the LSNDP is similar to supply chain optimization problems \cite{Beamon:1998} such as the Logistics Network Design Problem (LNDP) \cite{Srivastava:2008} and the Supply Chain Network Design Problem (SCNDP) \cite{Melo:2009}, which also aim to optimize the sourcing and fulfillment of orders through a multi-echelon distribution network. Nevertheless, supply chain optimization problems primarily focus on strategic planning decisions such as facility location \cite{Cordeau:2006,Cheong:2007}, whereas the LSNDP assumes the facilities in a network, and their capacities, have already been established. In addition, the LSNDP seeks to optimize transportation costs that are generally not estimated accurately in supply chains optimization problems, since these problems usually cover long-term planning horizons that are discretized in monthly or yearly periods. For more details regarding the positioning of the LSNDP with regard to the SNDP and the LNDP, we refer the interested reader to the literature review proposed in \cite{Belieres:2019}

To solve the LSNDP, we propose a solution approach based on Benders decomposition. Proposed by J.F. Benders \cite{Benders:1962} in 1962, Benders decomposition is designed for large-scale optimization problems that possess a certain feature. Namely, that the problem involves variables, often referred to as complicating variables, that, when fixed,  yield (sub)problems that are significantly less difficult to solve computationally-speaking. Instead of solving one large Mixed Integer Linear Programs (MILP), Benders decomposition separates the problem into a master problem that is solved to determine the values of the complicating variables, and one or multiple subproblems formulated by fixing the complicating variables to the values of the most recent master problem solution. When subproblems are feasible, and variable values from solutions to the master problem and subproblems can be used to form a provably optimal solution to the original problem, the algorithm stops. When they cannot, Benders cuts derived from the subproblems are used to strengthen the master problem, and the process repeats.

The Benders decomposition algorithm has been used to solve a wide range of optimization problems, including network design problems \cite{Costa:2005}. However, only implementing the steps of Benders decomposition that guarantee convergence often yields an algorithm that requires a substantial amount of time and memory before converging. As a result, many enhancement strategies have been proposed to accelerate the method. Instead of reviewing those here, we refer to \cite{Rahmaniani:2017}, which provides a taxonomy of these acceleration techniques. That taxonomy includes advanced algorithms for solving the master problem and the subproblem (e.g. \cite{Benoist:2002,Cordeau:2001}), approaches to generate solutions and constraints from those solutions (e.g. \cite{Costa:2012,Magnanti:1981,Rei:2009}), and advanced decomposition strategies, i.e. (e.g. \cite{Gendron:2016}). Contrary to standard decompositions, where all the linking constraints and non complicating variables are projected out, advanced decomposition strategies retain part of the subproblem information in the master problem, which can considerably improve the convergence of the algorithm. Indeed, solution algorithms based on standard Benders decomposition often struggle computationally-speaking, as partitioning complicating and non-complicating variables can lead to hide many of the underlying structures inherent to the original optimization problem. These algorithms instead have to iteratively re-discover representations of that structure by repeatedly solving master and sub-problems. In addition, many optimization solvers have effective techniques for detecting and exploiting known structures in optimization problems. On the other hand, advanced decomposition strategies retain subproblem information in the master problem, which reinforces the bounds it produces and allows to reduce the number of Benders iterations.

The exact Benders decomposition algorithm introduced in this study relies on such an advanced decomposition strategy: the  \textit{Partial Benders Decomposition} proposed by Crainic et al. \cite{Crainic:2014,Crainic:2016}. Proposed in the context of solving two-stage stochastic programs, the partial decomposition technique consists in including explicit information from the scenario sub-problems in the master problem. Through an extensive computational study, the authors demonstrate that Benders schemes based on partial decompositions rather than standard decompositions yields significant improvements in terms of stability of the solution process and computational time. Recently, partial decomposition techniques has also been applied in the context of solving deterministic programs. In that case, the information used to strengthen the master problem is obtained by aggregating data related to the subproblem. Fontaine et al. \cite{Fontaine:2017} address a multi-modal variant of the SNDP and propose a Benders decomposition algorithm wherein the master problem is reinforced by considering an aggregation of capacity over compartments of each service. Belieres et al. \cite{Belieres:2019} propose a similar type of enhanced Benders decomposition algorithm for solving the LSNDP. In this approach, the master problem is reinforced with variables and constraints that model the need to route a super-product derived from the aggregation of all the products to be transported. An extensive computational study shows this new master problem yields significantly stronger lower bounds.

\section{Problem definition and mathematical model}
\label{sec:prob}
\noindent
In this section, we first define the problem considered in this paper. We then present our mathematical formulation of that problem.

\subsection{Problem definition}
\label{subsec:prob_defn}
\noindent 
We study a tactical transportation planning problem that is inspired by the operations of a distribution company supporting restaurant supply chains. Specifically, the distribution company plans product distribution from suppliers to customers over a fixed planning horizon. For more details on the practical application, we refer the interested reader to \cite{Belieres:2020}.

On the supply side, products are provided by a set of suppliers in different locations, and are classified into product families. Examples of product families in restaurant supply chains are meats, beverages, fruits, vegetables, and paper products. Each supplier specializes in a subset of product families and only provides products from those families in which it specializes. However, a supplier that specializes in a product family is not required to produce all the products in that family. For example, the fruit family may contain both bananas and apples but a fruit supplier may supply apples but not bananas. In the industrial setting that inspired this research there are seven product families and each supplier specializes in at most three of them. 


On the demand side, product deliveries are requested by customers, which are restaurants in the industrial problem that motivated this research. To facilitate their inbound logistics operations, each customer has a periodic schedule that specifies the days and time windows during which they can receive products. For example, restaurant A may request delivery of products every Monday between 6 and 7 a.m. Note that deliveries must occur during the specified time window, as early or delayed deliveries are not allowed. The types and quantities of products ordered do not have to be the same from one request to another. For example, restaurant A may request one pallet of bananas and one pallet of wine for the first Monday of the month, and one pallet of broccoli for the second Monday of the month. 

Restaurants place product orders that specify the delivery day and time window but not the supplier for each product. Thus, the distribution company must determine how to source orders. This in turn implies that a single customer order that consists of multiple products may be sourced from multiple suppliers. For example, to fulfill restaurant A's order for the first Monday of the month, the distribution company may decide to ship one pallet of bananas from supplier 1 and one pallet of wine from supplier 2. We presume that restaurant orders are known far enough in advance that suppliers can design production plans to avoid stockouts.


The distribution company can ship products directly from suppliers to customers. However, products are small relative to vehicle capacity. Thus, to consolidate orders and increase vehicle utilization, the distribution company may instead transport products through a distribution network that links suppliers with customers. Intermediate terminals within this distribution network are referred to as \textit{Warehouses}. Each warehouse can store a limited amount of product, but doing so incurs a per-unit, per-unit-of-time cost. A vehicle dispatched from a supplier to a warehouse or from a warehouse to another warehouse can carry products destined for different customers. However, for this industrial partner, a vehicle dispatched to a customer can only transport products intended for that customer. Figure \ref{fig:supply_chain_network} represents an example of such a distribution network, where $S_{x}$ nodes are supplier facilities, $C_{x}$ nodes are customer locations, and $W_{x}$ nodes are warehouses within the distribution network.

 \begin{figure}[H]
          \centering \includegraphics[scale=0.05]{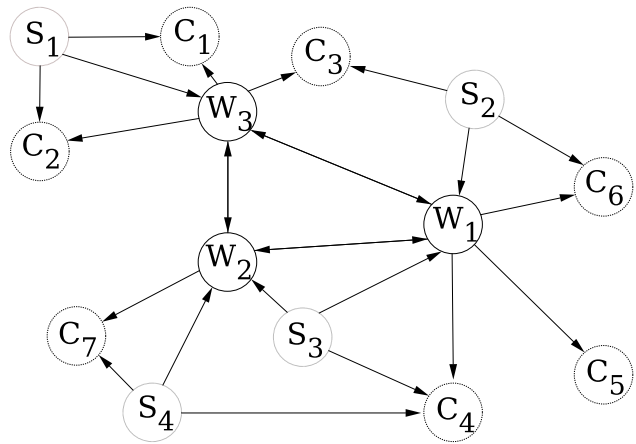}
          \caption{Distribution network}
          \label{fig:supply_chain_network}
    \end{figure}

Transportation between a pair of terminals is denoted as a \textit{service}. A service is defined by a departure terminal, a departure time, an arrival terminal, and an arrival time. Using a service requires the assignment of  transportation units to that service. Each unit has a fixed capacity and a cost. Thus, the distribution company determines the services to execute and the number of vehicles to dispatch on each service. 

We focus on a situation wherein the distribution company is an asset-less logistics service provider and outsources transportation operations to a third party carrier. The distribution company develops a transportation plan for moving products and communicates the resulting needs for point-to-point transportation moves to the carrier. Therefore, the distribution company does not make vehicle fleet management decisions such as planning empty moves. Similarly, the distribution company assumes that the carrier's fleet is large enough to satisfy the services it wants to have performed. In total, the third-party logistics company seeks to determine the shipment origin of each product requested, as well as the services and transportation capacities needed to support deliveries. Their objective is to minimize the overall cost of product transportation and storage.

\subsection{Mathematical model}
\label{subsec:prob_form}
\noindent
We model the supply chain with a directed network, $\flatnetwork = (\flatnodes{} ,\flatarcs)$ with node set $\flatnodes$ and arc set $\flatarcs.$ The node set $\flatnodes$ is composed of nodes that model suppliers $\suppliers$, customers $\customers$, and warehouses $\warehouses$. The arc set $\flatarcs$ is composed of arcs that model transportation between pairs of nodes. Due to the multi-echelon structure of the supply chain, $\flatarcs$ does not contain arcs that model transportation to a supplier, or arcs that model transportation from a customer. A transportation arc can only be defined in one of the following situations: i) from a supplier to a warehouse; ii) from a supplier to a customer; iii) from a warehouse to another warehouse; iv) from a warehouse to a customer. With each arc $a =(i,j) \in \flatarcs$ is associated a travel time $\transittime  \in \mathbb{R}^{+*}$, a per unit of flow cost $\ltcost \in \mathbb{R}^{+*}$, a vehicle capacity $\truckcapacity$, and a fixed cost per vehicle, $\fcost \in \mathbb{R}^{+*}$.  

The set of all products ordered by customers is denoted by $\products.$ The set of products manufactured by supplier $i$ is denoted by $\products^i$. We denote the set of product families as $\families=\{\families_1,...,\families_{|\families| } \}$. As each product belongs to a single family, product families form a partition of the product set. Formally, we have that $\forall (i,j) \in \{1,...,|\families|\}^2, i \neq j, \families_i \cap \families_j = \emptyset,$ and $\bigcup\limits_{i \in \{1,...,|\families| \}} \families_i = \products$.

The distribution company determines a transportation plan over a time horizon of length $\timehorizon$. To model the time dimension, we consider a time-expanded network $\expnetwork = \left(\expnodes, \expharcs \cup \exptarcs \right)$ derived from the network $\flatnetwork$. $\expnetwork$ contains a time-expanded node $(i, t) \in \expnodes,$ for each $i \in \flatnodes$ and $t \in \timehorizon$. Thus, time-expanded nodes model either  time-expanded suppliers $\suppliers_\timehorizon$, time-expanded customers $\customers_\timehorizon$ or time-expanded warehouses $\warehouses_\timehorizon$. Arcs in $\expharcs$ model product storage at warehouses. Thus, for each $i \in \warehouses$ and each $t \in [1, |\timehorizon|-1]$, there is a time-expanded arc $((i, t), (i, t+1))$ in $\expharcs$ with a per-unit-of-flow cost $c_{ii}$ and a storage capacity ${\warehouse}lim_i$. Arcs in $\flatarcs_\timehorizon$ model transportation between different locations, taking account of the travel time. To construct these arcs, for each $(i, j) \in \flatarcs$ and each time $t \in \timehorizon$ such that $t+t_{ij} < |\timehorizon|$, we build a time-expanded arc $((i, t), (j, t+t_{ij}))$. Therefore, an arc $((i, t),(j, t+t_{ij}))$ in $\exptarcs$ models the flow of products departing from $i$ at time $t$, and arriving to $j$ at time $t+t_{ij}$. We note that before creating the time-expanded graph $\expnetwork,$ arc travel times $t_{ij}$ may need to be modified to ensure that arcs $(i,j) \in \flatarcs$ can be mapped to arcs of the form $((i, t), (j, t+t_{ij}))$. As an example, for an arc $(i, j)$ with $t_{ij}=5h$, a 3-hours discretization requires to round up $(i, j)$ travel time to the nearest multiple of the time-step, i.e. 6 hours. Consequently, discretizing the planning horizon may introduce travel time approximations that tend to reduce with finer time-steps. 

Thus, we formulate the Logistics Service Network Design problem defined over a time-expanded graph $\expnetwork$. For each transportation arc $\tarc \in \exptarcs$, the integer variable $\trucksvariables$  corresponds to the number of trucks dispatched. For each transportation/holding arc $\tarc \in \expharcs \cup \exptarcs$, the continuous variable $\flowstvariables$ models the flow of product $p$ on that arc.  Note that variable  $\flowstvariables$ is not defined if $\tarc$ is a transportation arc originating from a supplier $i$ that does not manufacture product $p$ (i.e. $p \notin  \products^i$). Customer demands are represented by $\demand$, which is the amount of product $\product \in \products$ requested by customer $c \in \customers$ to be delivered at time $t \in [1, |\timehorizon|]$. 

It should be noted that the formulation of the LSNDP below explicitly references product families. This is technically not necessary. However, as we leverage this problem characteristic in the Benders decomposition-based algorithm we present later in this paper, we include them in this formulation. Ultimately, the LSNDP can be stated as:

\begin{align}
\mbox{minimize } \sum\limits_{\tarc \in \exptarcs}     \fcost \trucksvariables  \hspace{0.2cm} +  \hspace{0.2cm}  \sum\limits_{\tarc \in \exptarcs}  \sum\limits_{\product \in \products}    \ltcost \flowstvariables \hspace{0.2cm} +  \hspace{0.2cm}  \sum\limits_{\harc \in \expharcs} \sum\limits_{\product \in \products}   \lhcost \flowshvariables  
\label{objectiveFULL}
\end{align}
subject to the following constraints : 

\vspace{-0.2cm} 
\begin{align}
\sum\limits_{\tarc \in \exptarcs \cup \expharcs} \flowstvariables- \sum\limits_{((j,\time'),(l,\time'')) \in \exptarcs \cup \expharcs} \flow^{\product\time'\time''}_{jl} = 0, \hspace{0.5cm} \forall (j,\time') \in \warehouses_{\timehorizon}, \forall \product \in \families_i , \forall \families_i \in \families 
\label{flowConservationConstraintFULL}
\end{align}

\vspace{-0.2cm} 
\begin{align}
\sum\limits_{\tarc \in \exptarcs} \flow^{\product\time\time'}_{ij} \geq \d^\product_{j\time'}, \hspace{0.5cm} \forall (j,\time') \in \customers_{\timehorizon}, \forall \product \in \families_i , \forall \families_i \in \families 
\label{demandConstraintFULL}
\end{align}

\vspace{-0.3cm} 
\begin{align}
\sum\limits_{\product \in \products} \flowshvariables \leq \warehousecapacity,  \hspace{0.5cm} \forall \harc \in \expharcs
\label{warehouseCapacityConstraintFULL}
\end{align}

\vspace{-0.3cm} 
\begin{align}
\sum\limits_{\product \in \products} \flowstvariables \leq \truckcapacity \trucksvariables,  \hspace{0.5cm} \forall \tarc \in \exptarcs
\label{truckActualizationConstraintFULL}
\end{align}

\vspace{-0.3cm} 
\begin{align}
\flowstvariables \in  \mathbb{R}^{+},  \hspace{0.5cm} \forall \tarc \in \exptarcs \cup \expharcs,  \hspace{0.15cm}  \forall \product \in \products^i
\label{flowsDomainFULL}
\end{align}

\vspace{-0.4cm} 
\begin{align}
\trucksvariables \in  \mathbb{N}^{+},  \hspace{0.5cm} \forall \tarc \in \exptarcs 
\label{trucksDomainFULL}
\end{align}

The LSNDP seeks to minimize the transportation plan overall cost \eqref{objectiveFULL}, which is the sum of trucks fixed costs on transportation arcs (first term), and flow linear costs on transportation arcs (second term) and holding arcs (third term). Flow feasibility is ensured by the two first constraints. Constraints \eqref{flowConservationConstraintFULL} enforce the flow conservation at each warehouse. Constraints \eqref{demandConstraintFULL} enforce that each customer demands are fulfilled. Constraints \eqref{warehouseCapacityConstraintFULL} limit the total amount of product stored by each warehouse. Constraints \eqref{truckActualizationConstraintFULL} ensure that a sufficient number of trucks are allocated to transport products. Variable domains are defined by \eqref{flowsDomainFULL} and \eqref{trucksDomainFULL}.

\section{Benders decompositions for the LSNDP}
\label{sec:decomposition}
\noindent
In this paper, we present a Partial Benders Decomposition-based algorithm that intelligently switches between master problems strengthened with different aggregated subproblem informations. As such, this section focuses on different master problems for solving the LSNDP with a Benders decomposition-based algorithm. We begin this section by presenting a Benders decomposition for the problem based upon the standard master problem. We then introduce a new master problem that is reinforced with variables and constraints that model the need to route $\K$ super-products derived from a $\K$-partition of the product set. We study the theoretical properties of such a master problem, including establishing that it is a relaxation of the original problem. We demonstrate that the way products are partitioned impacts the strength of the resulting master problem. Lastly, we discuss the effect that the number of super-products $\K$ has on the strength of the resulting master problem. 


\subsection{Standard Benders decomposition for the LSNDP}
\label{subsec:straightforward}
\noindent

In the context of the LSNDP, a standard Benders decomposition yields a master problem that allocates vehicles on transportation arcs. The resulting subproblem aims to route products from suppliers to customers, while ensuring that the total flow on each transportation arc does not exceed the capacity defined by the solution to the master problem. Based on the set of extreme rays of the subproblem dual polyhedron, $\extremerays$, and the set of extreme points of the subproblem dual polyhedron, $\extremepoints$, the standard master problem, \textbf{SMP}, is formulated as follows:

\vspace{-0.2cm}
	
\begin{align} \text{min} \hspace{0.3cm} \sum\limits_{\tarc \in \exptarcs}     \fcost \trucksvariables + z 
\label{objectiveSMP}
\end{align}
   
\vspace{-0.4cm}

\begin{align}
0 \geq 
\sum\limits_{(c,t) \in \customers_\timehorizon} \sum\limits_{p \in \products} \demand \extremeray_{ct}^p 
+ \sum\limits_{\harc \in \expharcs} \warehousecapacity \extremeray_{ii}^{tt'}
- \sum\limits_{\tarc \in \exptarcs} \truckcapacity \extremeray_{ij}^{tt'}  y_{ij}^{tt'}, \hspace{0.3cm} \forall \extremeray \in \extremerays
\label{feasibilityConstraintSMP}
\end{align}

\vspace{-0.5cm}

\begin{align}
z \geq 
\sum\limits_{(c,t) \in \customers_\timehorizon} \sum\limits_{p \in \products} \demand \extremepoint_{ct}^p 
+ \sum\limits_{\harc \in \expharcs} \warehousecapacity \extremepoint_{ii}^{tt'}
- \sum\limits_{\tarc \in \exptarcs} \truckcapacity \extremepoint_   {ij}^{tt'}  y_{ij}^{tt'}, \hspace{0.3cm} \forall \extremepoint \in \extremepoints
\label{optimalityConstraintSMP}
\end{align}

\vspace{-0.6cm}
\begin{align}
\trucksvariables \in  \mathbb{N}^{+},  \hspace{0.5cm} \forall \tarc \in \exptarcs 
\label{trucksDomainSMP}
\end{align}

\vspace{-0.6cm}

\begin{align}
\z \in  \mathbb{R}^{+} 
\label{zDomainSMP}
\end{align}

The objective function, \eqref{objectiveSMP}, computes costs related to the allocation of the vehicle fleet as well as an estimate of product routing costs. Constraints \eqref{feasibilityConstraintSMP} and \eqref{optimalityConstraintSMP} are respectively the feasibility and optimality standard Benders cuts added dynamically after solving the subproblem. 

Based on an allocation of vehicles $\trucksalloc$, the subproblem \textbf{SP}$(\bar{y})$ is formulated as follows:

\vspace{-0.3cm}

\begin{align} 
\text{min} \hspace{0.05cm} \sum\limits_{\tarc \in \exptarcs} \sum\limits_{\product \in \products}    \ltcost \flowstvariables \hspace{0.05cm} +  \hspace{0.05cm} \sum\limits_{\harc \in \expharcs}   \sum\limits_{\product \in \products}  \lhcost \flowshvariables
\label{objectiveSP}
\end{align}

\begin{center}
   \eqref{flowConservationConstraintFULL}-\eqref{demandConstraintFULL}-\eqref{warehouseCapacityConstraintFULL}
\end{center}

\vspace{-0.5cm} 

\begin{align}
\sum\limits_{\product \in \products} \flowstvariables \leq \truckcapacity \trucksvariablesalloc,  \hspace{0.5cm} \forall \tarc \in \exptarcs
\label{fixedCapacityConstraintSP}
\end{align}

		\vspace{-0.55cm}

\begin{align}
\flowstvariables \in  \mathbb{R}^{+},  \hspace{0.5cm} \forall \tarc \in \exptarcs \cup \expharcs,  \hspace{0.15cm}  \forall \product \in \products^i
\label{flowsDomainSP}
\end{align}

Given a vehicle allocation $\bar{\truck}$, the subproblem seeks to route products from suppliers to customers at a minimal cost, while ensuring that the flow variables respect constraints \eqref{flowConservationConstraintFULL}-\eqref{demandConstraintFULL}-\eqref{warehouseCapacityConstraintFULL} from the original problem. Constraint \eqref{fixedCapacityConstraintSP} is defined on each transportation arc and imposes that the total flow cannot exceed the available capacity.


\subsection{A master problem based on super-products}
\label{subsec:reinforce}
\noindent
 The principle behind \textit{Partial Benders Decomposition} \cite{Crainic:2014,Crainic:2016} is to strengthen the master problem with information derived from the subproblem. Belieres et al. \cite{Belieres:2019} apply this principle in the context of solving the LSNDP by reinforcing the standard master problem with variables and constraints related to  routing  a single super-product obtained by aggregating all products $\product \in \products$. We extend this approach to multiple super-products. In particular, given a partition of the product set into subsets, we propose a master problem obtained by aggregating each product subset into its own super-product. We next present this master problem.

Let $\{ \products_1,...,\products_\K \}$ be a $\K$-partition of the product set $\products$. We create the super-product $\superproduct_k$ by aggregating all products included in subset $\products_k$. For a given customer $\customer$ at time $\time$, the demand of super-product $\superproduct_k$ is the sum of the demands for all products in $\products_k$, i.e. $\D^{\superproduct_k}_{\customer\time} = \sum\limits_{\product \in \products_k} \d^\product_{\customer\time}$. A super-product flow variable $\flow^{\superproduct_k \time\time'}_{ij}$ is defined for each arc $\tarc$ and each product $\product \in \products_k$ such that a flow variable $\flow^{\product\time\time'}_{ij}$ is defined in the LSNDP. Thus, for a supplier $i$ that does not supply any product of $\products_k$ (i.e. $\products^i \cap \products_k  = \emptyset$), and for all arcs originating from $i$, the continuous variable $\flow^{\superproduct_k \time\time'}_{ij}$ is not defined.
We let $\superproducts$ denote the set of all such super-products and refer to the resulting master problem as the \textbf{K-EMP}.

Figures \ref{before_aggreg_K} and \ref{after_aggreg_K} illustrate an example. The original problem is depicted in Figure \ref{before_aggreg_K}. A single customer requests one unit of $\product_1, \product_2, \product_3$ and $\product_4$. Supplier $\supplier_1$ manufactures products $\product_1$ and $\product_2$, while $\supplier_2$ manufactures $\product_2$ and $\product_3$, and $\supplier_3$ manufactures $\product_3$ and $\product_4$. Given a 2-partition of the products: $\{ \products_1 = \{\product_1,\product_2\}, \products_2 = \{\product_3,\product_4\} \}$ , we aggregate the products into two super-products $\superproduct_1$ and $\superproduct_2$. The aggregated problem is depicted in Figure \ref{after_aggreg_K}.  As $\customer$ requests one unit of  $\product_1$ and one unit of $\product_2$ in the original problem, $\customer$ request two units of super-product $\superproduct_1$ in the aggregated problem. As $\supplier_1$ and $\supplier_2$ manufacture at least one product of $\products_1$ in the original problem, they manufacture $\superproduct_1$ in the aggregated problem as well. On the contrary, as $\supplier_3$ does not manufacture any product of $\products_1$ in the original problem, it does not manufacture $\superproduct_1$ in the aggregated problem. A similar reasoning applies to super-product $\superproduct_2$ and products $\product_3$ and $\product_4$.

\begin{figure}[h!]
\RawFloats
       \begin{minipage}[b]{0.48\linewidth}
          \centering \includegraphics[scale=0.075]{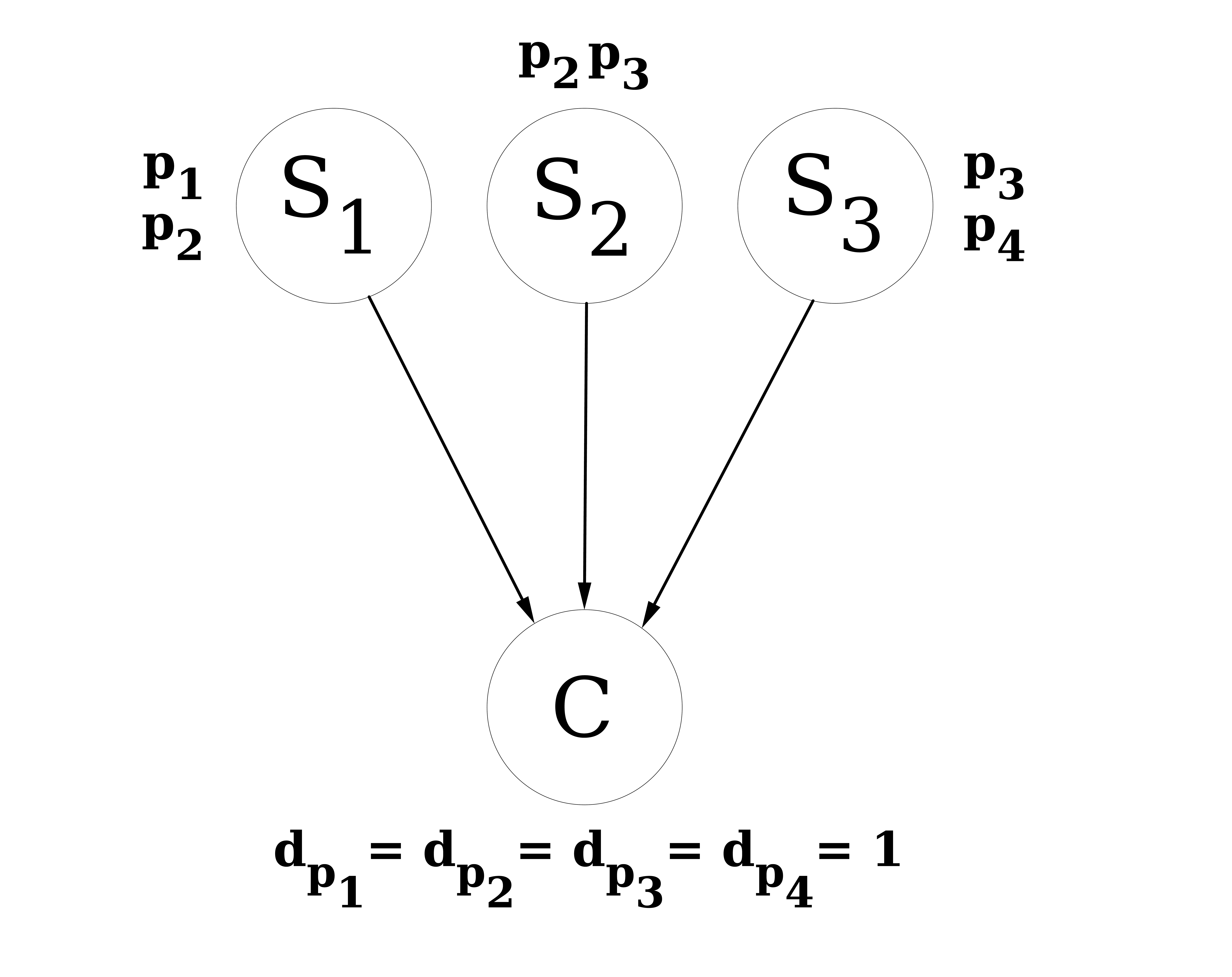}
          \caption{Before the aggregation of products}
          \label{before_aggreg_K}
       \end{minipage}\hfill 
       \begin{minipage}[b]{0.48\linewidth}   
          \centering \includegraphics[scale=0.075]{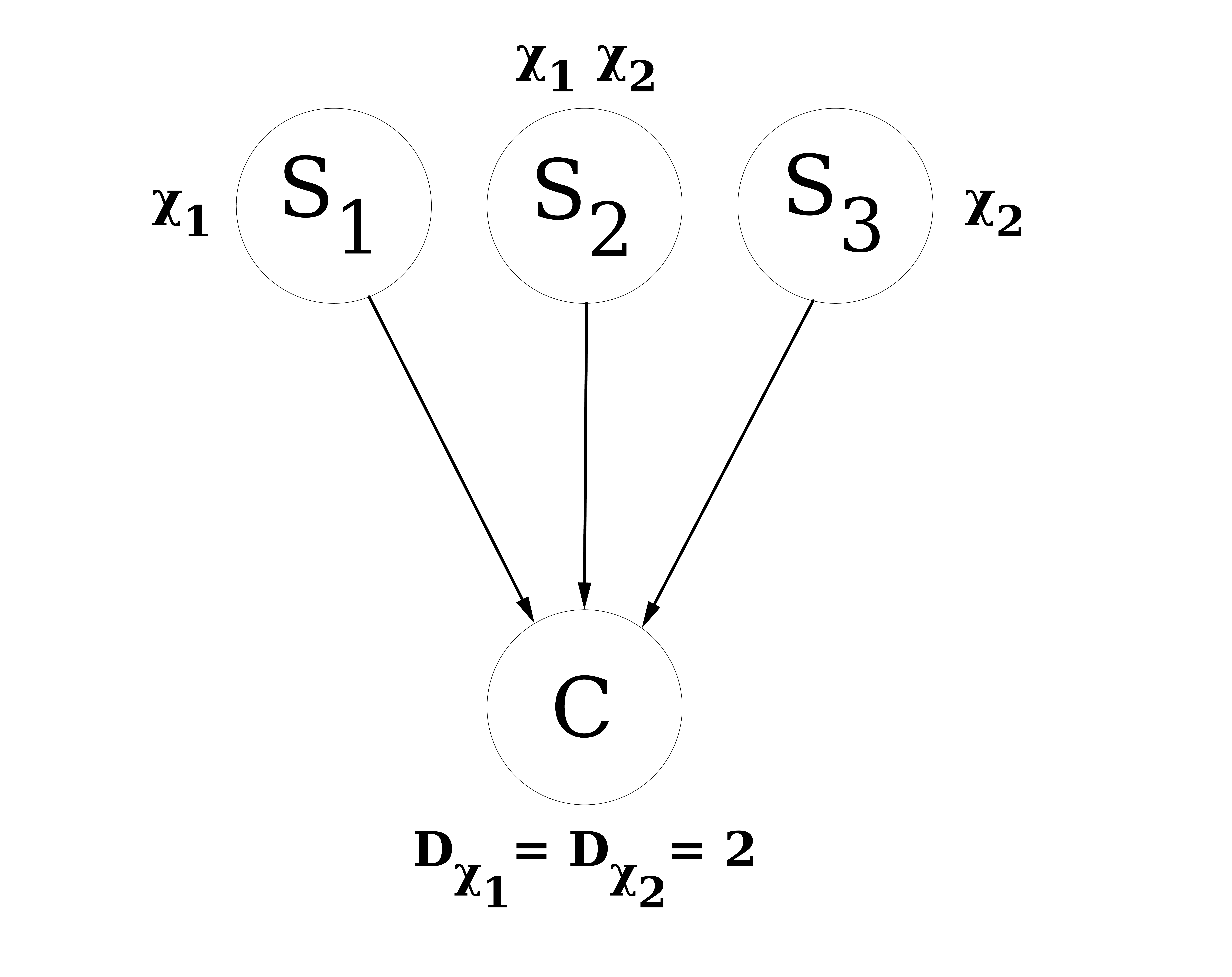}
          \caption{After aggregating $\product_1$ with $\product_2$, and $\product_3$ with $\product_4$}
          \label{after_aggreg_K}
       \end{minipage}
    \end{figure}

Note that the aggregation of products is an approximation of the original problem. For example, a supplier may manufacture a super-product $\superproduct_k$ in the \textbf{K-EMP} whereas it does not manufacture all products in $\products_k$ in the original problem.  
In  our example,  $\supplier_2$ in the aggregated problem manufactures all super-products and thus  can satisfy all customer demands. However, in the original problem $\supplier_2$ only manufacture products $\product_2$ and $\product_3$, and cannot satisfy demands for products $\product_1$ and $\product_4$. Thus, not every feasible solution to the \textbf{K-EMP} can be converted to a feasible solution for the LSNDP.

The \textbf{K-EMP} allocates vehicle capacity on transportation arcs in order to satisfy the routing of the K super-products. It is formulated as follows:

\vspace{-0.5cm}
	
\begin{align} \text{min} \hspace{0.3cm} \sum\limits_{\tarc \in \exptarcs}     \fcost \trucksvariables + \z 
\label{objectiveKEMP}
\end{align}
   
\vspace{-0.55cm}

\begin{align}
\sum\limits_{\tarc \in \exptarcs \cup \expharcs} \flowstkvariablesaggregated - \sum\limits_{((j,\time'),(l,\time''))\exptarcs \cup \expharcs}  \flow^{\superproduct_k \time'\time''}_{jl} = 0, \hspace{0.5cm} \forall (j,\time') \in \warehouses_{\timehorizon}, \forall \superproduct_k \in \superproducts
\label{flowConservationConstraintKEMP}
\end{align}

\vspace{-0.5cm}

\begin{align}
\sum\limits_{\tarc \in \exptarcs}  \flowstkvariablesaggregated \geq \D^{\superproduct_k}_{j\time'}, \hspace{0.5cm} \forall (j,\time') \in \customers_{\timehorizon}, \forall \superproduct_k \in \superproducts
\label{demandConstraintKEMP}
\end{align}

\vspace{-0.55cm}

\begin{align}
\flowshvariablesaggregated \leq \warehousecapacity,  \hspace{0.5cm} \forall \harc \in \expharcs
\label{warehouseCapacityConstraintKEMP}
\end{align}

\vspace{-0.55cm}

\begin{align}
\sum\limits_{\superproduct_k \in \superproducts} \flowstkvariablesaggregated \leq \truckcapacity \trucksvariables,  \hspace{0.5cm} \forall ((i,\time),(j,\time')) \in \exptarcs
\label{truckActualizationConstraintKEMP}
\end{align}

\vspace{-0.65cm}

\begin{align}
\z \geq \sum\limits_{\tarc \in \exptarcs}   \sum\limits_{\superproduct_k \in \superproducts}  \ltcost \flowstkvariablesaggregated  +  \sum\limits_{\harc \in \expharcs}  \sum\limits_{\superproduct_k \in \superproducts}  \lhcost \flowshkvariablesaggregated
\label{zConstraintKEMP}
\end{align}

\vspace{-0.25cm}

\begin{center} \eqref{feasibilityConstraintSMP}-\eqref{optimalityConstraintSMP}  \end{center}

\vspace{-1cm}

\begin{align}
\flowstkvariablesaggregated \in  \mathbb{R}^{+},  \hspace{0.5cm} \forall ((i,\time),(j,\time')) \in \exptarcs \cup \expharcs,  \hspace{0.15cm}  \forall \superproduct_k \in \superproducts, \hspace{0.15cm} \products^i \cap \products_k  \neq \emptyset
\label{flowsDomainKEMP}
\end{align}

\vspace{-0.9cm}

\begin{align}
\trucksvariables \in  \mathbb{N}^{+},  \hspace{0.5cm} \forall ((i,\time),(j,\time')) \in \exptarcs 
\label{trucksDomainKEMP}
\end{align}

\vspace{-0.9cm}

\begin{align}
\z \in  \mathbb{R}^{+}
\label{zDomainKEMP}
\end{align}

The objective function is the same as for the \textbf{SMP}. Constraints \eqref{flowConservationConstraintKEMP} enforce the flow conservation of each super-product at each warehouse. Constraints \eqref{demandConstraintKEMP} ensure that, for each customer, each super-product demand is fulfilled. Constraint \eqref{warehouseCapacityConstraintKEMP}  limits the total amount of super-product stored by each warehouse. Constraints \eqref{truckActualizationConstraintKEMP} ensure that enough vehicle capacity is allocated to support the flows of super-products. Constraint \eqref{zConstraintKEMP} bounds the flows cost approximation $z$. Constraints \eqref{feasibilityConstraintSMP} and \eqref{optimalityConstraintSMP} are the standard Benders cuts added dynamically after solving the subproblem. Constraints \eqref{flowsDomainKEMP}, \eqref{trucksDomainKEMP}, and \eqref{zDomainKEMP} define the decision variables and their domain. 

For Benders to converge it must solve a master problem that is a relaxation of the original problem. While \cite{Belieres:2019} prove that \textbf{K-EMP} is a relaxation when $\K = 1$, the following is true for general $\K$. We defer the proof to the Appendix, Section \ref{sec:appendix}.

\begin{theorem}\label{theorem1}
The $\K$-enhanced master problem, \textbf{K-EMP}, is a relaxation of the Logistics Service Network Design problem, LSNDP.
\end{theorem}

\subsection{Creating a strong super-product based master problem}
\label{subsec:effect_paritition}
\noindent
For a given $\K$, there are different ways to partition the set of products into $\K$ subsets of products, and thus different master problems, \textbf{K-EMP}. Different partitions of the product set may result in master problems that approximate the original problem to different degrees. We first identify a specific case where an appropriate partitioning of the product set gives a master problem that is equivalent to the original problem. Based on this observation, we define a metric that evaluates the potential benefit of aggregating a pair of products on the resulting Benders algorithm, and thus the potential benefit of including those products in the same subset of the $\K$-partition.


We first explain this specific case and then discuss the intuition behind why it yields a \textbf{K-EMP} that is equivalent to the original problem. Consider an instance of the original problem wherein each supplier that provides any product $\product \in \products_k$ can also provide all other products included in $\products_k$. In this case, the products of $\products_k$ can be aggregated without loss of information. Furthermore, given a $\K$-partition of the products $\{ \products_1,...,\products_\K \}$, if for all subsets of $\products_k$, each supplier that provides any product $\product \in \products_k$ can also provide all other products included in $\products_k$, then all product subsets of the $\K$-partition can be aggregated without loss of information. More precisely, one can prove that in this specific case  the resulting \textbf{K-EMP} is equivalent to the original problem. We next present the intuition behind a proof of this claim.


Figures \ref{equivalence_kemp_before} and \ref{equivalence_kemp_after} illustrate an example of equivalence between the \textbf{K-EMP} and the LSNDP. The original problem is depicted in  Figure \ref{equivalence_kemp_before}. Both products $\product_1$ and $\product_2$ are offered by $\supplier_1$ and $\supplier_2$. Aggregating $\product_1$ and $\product_2$ into super-product $\superproduct_1$ induces no loss of information. Indeed, if we aggregate $\product_1$ with $\product_2$ (see Figure \ref{equivalence_kemp_after}), the demand of super-product $\superproduct_1$ can be satisfied by $\supplier_1$ or $\supplier_2$. As the demand of super-product $\superproduct_1$ sums the demands for $\product_1$ and $\product_2$, the \textbf{K-EMP} assumes that both $\supplier_1$ and $\supplier_2$ can satisfy the demands for $\product_1$ and $\product_2$. Yet, in the original problem, $\supplier_1$ and $\supplier_2$ provide both $\product_1$ and $\product_2$. Thus, we can aggregate $\product_1$ and $\product_2$ without loss of information. Similarly, the aggregation of $\product_3$ and $\product_4$ into super-product $\superproduct_2$ induces no loss of information. Therefore, the $2$-partition  $\{ \products_1 = \{\product_1,\product_2\}, \products_2 = \{\product_3,\product_4\} \}$  yields a \textbf{K-EMP} that is equivalent to the LSNDP. 


\begin{figure}[h!]
\RawFloats
       \begin{minipage}[b]{0.45\linewidth}
          \centering \includegraphics[scale=0.075]{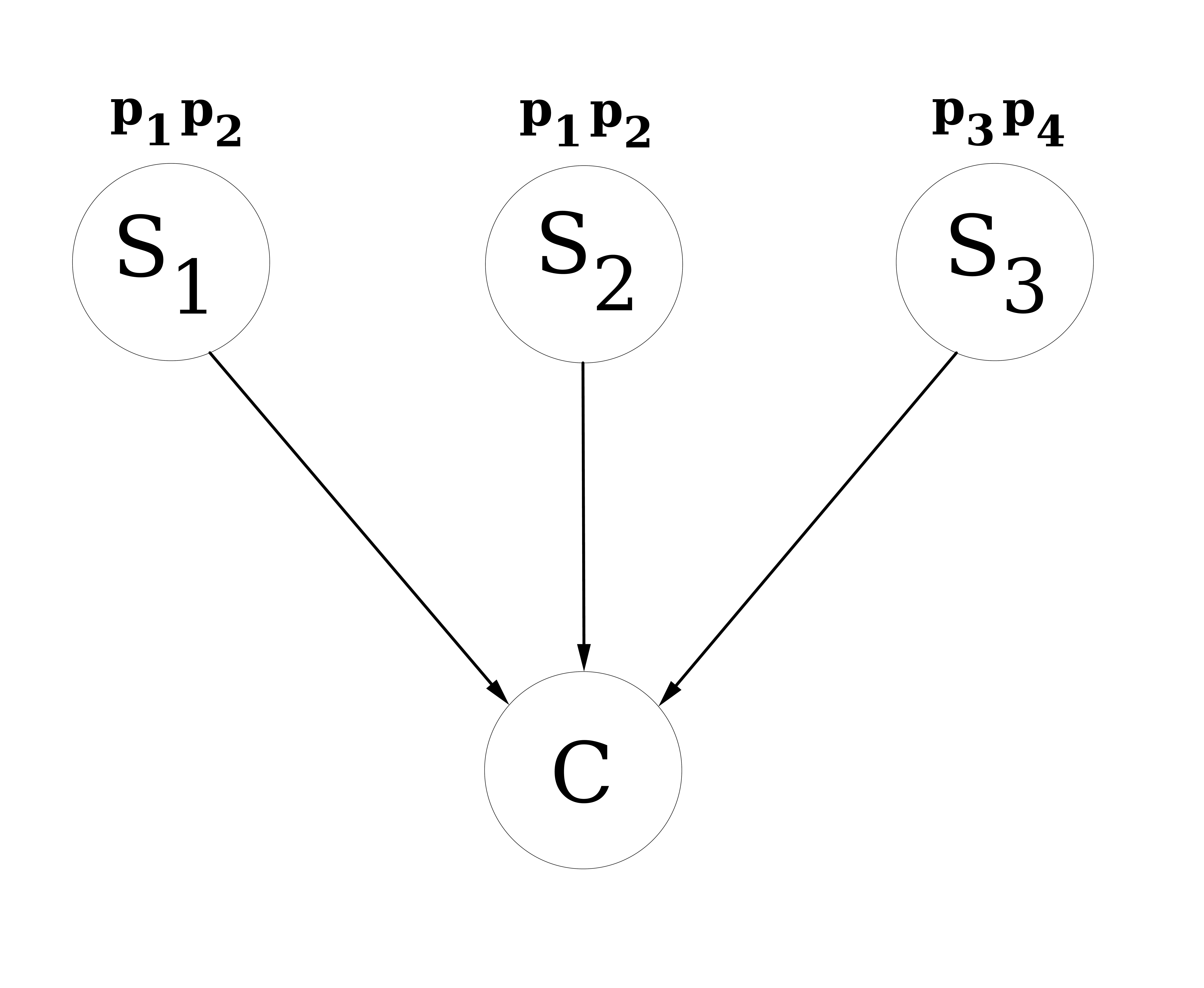}
          \caption{Before the aggregation of products}
          \label{equivalence_kemp_before}
       \end{minipage}\hfill 
       \begin{minipage}[b]{0.45\linewidth}   
          \centering \includegraphics[scale=0.075]{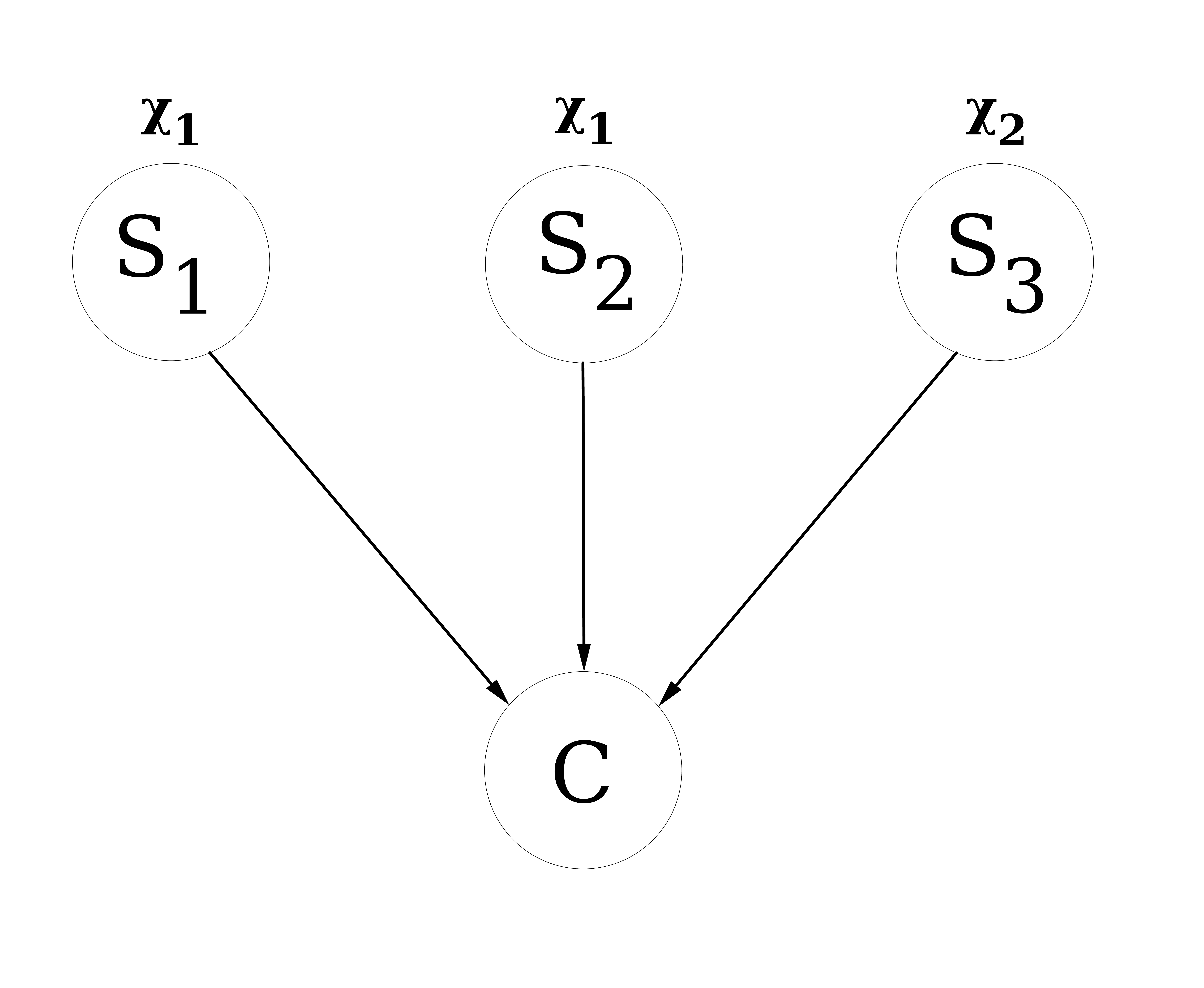}
          \caption{After aggregating $\product_1$ with $\product_2$, and $\product_3$ with $\product_4$}
          \label{equivalence_kemp_after}
       \end{minipage}
    \end{figure}

Formally, we have the following theorem, whose proof is included in the Appendix, Section \ref{sec:appendix}.
\begin{theorem}\label{theorem2}
Let $\{ \products_1,...,\products_\K  \}$ be a $\K$-partition of the products, and $\superproduct_k \in \superproducts, k \in \{1,...,\K\},$ the associated super-products. If for each $k \in \{1,...,\K\}$, each supplier $s \in \suppliers$ that provides any product $\product \in \products_k$ can also provide all other products included in $\products_k$, then the \textbf{K-EMP} is equivalent to the LSNDP.
\end{theorem}
Based on this theorem, we define a statistic that measures the potential benefit of including two products in the same product subset. For a product $\product$, we denote $\suppliers^\product$ as the set of suppliers that manufacture $\product$. We denote the matching rate of a pair of products $\product_i$ and $\product_j$ as:


$$ \matching (\product_i, \product_j) = \frac{|\suppliers^{\product_i} \cap \suppliers^{\product_j}|}{|\suppliers^{\product_i} \cup \suppliers^{\product_j}|} $$

This matching rate measures the fraction of suppliers of $\product_i$ or $\product_j$ that offer both products. We view this matching rate as an indicator of whether or not $\product_i$ and $\product_j$ should be aggregated. At one extreme, two products with a matching rate of one are provided by the same suppliers and thus can be aggregated without loss of information. At the other extreme, two products with a matching rate of zero have no common supplier and thus there is little to no perceived benefit on the Benders algorithm in aggregating them. We also define the matching rate of a product set $\products_i$ as the average matching rate of all pairs of products in that set, ie:

$$ \matching (\products_i) = \frac{2}{|\products_i|\times(|\products_i|-1)} \sum\limits^{i \leq |\products_i|}_{i=1} \sum\limits^{j \leq |\products_i|}_{j=i+1} \matching (\product_i, \product_j) $$

Similarly, this value comprised between zero and one indicates wether or not products of $\products_i$ should be aggregated. We note that, given $\{ \products_1,...,\products_\K  \}$ a $\K$-partition of the product set, the Theorem \ref{theorem2} holds as $\matching (\products_k) = 1, \forall k \in \{1,...,\K\}$. In the second phase of our Meta-PBD strategy, we propose an algorithm for determining such an "optimal" $\K$-partition of the product set.

\subsection{Relationship between $\K$ and the resulting \textbf{K-EMP} strength}
\label{subsec:effect_number}
\noindent
The number of super-products $\K$ used to formulate the \textbf{K-EMP} can vary between 0 and $|\products|$. When $\K=0$, the \textbf{K-EMP} corresponds to the standard Benders master problem, while when $\K=|\products|$ the master is strengthened with $|\products|$ super-products, with a one to one correspondance between super-products and products. In that case, the \textbf{K-EMP} is clearly equivalent to the original problem.  Similarly, the greater the number of super-products, the easier it is to meet the condition for applying Theorem \ref{theorem2}. Thus, one could conclude that the  higher the value of $\K$, the stronger the bound produced by solving the resulting super-product-based master problem.

However, this is not always true. In fact, solving a \textbf{(K+1)-EMP} may yield a weaker bound than solving a \textbf{K-EMP}. For example, consider again the LSNDP depicted in Figure \ref{equivalence_kemp_before}. We observed that the $2$-partition $\{ \products_1 = \{\product_1,\product_2\}, \products_2 = \{\product_3,\product_4\} \}$ results in a \textbf{2-EMP} that is equivalent to the original problem. Let us consider a $3$-partition $\{ \products'_1 = \{\product_1\}, \products'_2 = \{\product_2,\product_3\}, \products'_3 = \{\product_4\} \}$ that yields 3 super-products: $\superproduct_1$, $\superproduct_2$ and $\superproduct_3$. As the matching rate of $p_2$ and $p_3$ is null, this $3$-partition induces a loss of information. Thus, the $3$-partition defines a \textbf{3-EMP} that will produce a weaker bound than the \textbf{2-EMP} defined above.

A master problem with $\K$ super-products is not necessarily weaker than all master problems with $\K+1$ super-products. On the other hand, we have the following theorem that there always exists at least one master problem with $\K+1$ super-products such that the master problem based on $\K$ super-products is a relaxation. While we leave the proof of Theorem \ref{theorem3} to the Appendix (Section \ref{sec:appendix}), we note that the procedure for creating such a $\K+1$ master problem is straightforward. Namely, one needs to partition one of the product sets in the existing $\K$-partition.

\begin{theorem}\label{theorem3}
Given a $\K$-enhanced master problem, $\K \in \{1,...,|\products|-1\}$, there always exist a $\K+1$-enhanced master problem such that \textbf{K-EMP} is a relaxation of \textbf{K+1-EMP}
\end{theorem}


\section{Meta Partial Benders Decomposition}
\label{sec:dynamic}
\noindent

In the previous section, we established a condition wherein a \textbf{K-EMP} is equivalent to the original problem. It is noteworthy that using such a \textbf{K-EMP} as the master problem in a Benders decomposition-based algorithm would allow the algorithm to converge in a single iteration. However, it may require using a large number of super-products to meet this condition, making the resulting master problem computationally challenging to solve. Ultimately, the number of super-products used to formulate a \textbf{K-EMP} has both positive and negative impacts on the performance of a Benders decomposition-based algorithm. On the one hand, a larger number of super-products should yield a stronger master problem and allow the algorithm to converge in fewer iterations. On the other hand, a larger number of super-products should result in a master problem that is harder to solve and therefore the algorithm will spend more time executing each iteration. However, the magnitude of either impact cannot be accurately estimated before the execution of the resulting Benders decomposition-based algorithm. 

Thus, we propose an algorithm that executes a Benders decomposition-based algorithm on different master problems with different numbers of super-products. We refer to this algorithm as \textit{Meta Partial Benders Decomposition} (Meta-PBD). Meta-PBD operates in two phases. Phase I has a time limit set to $\timelimitphaseone$ units of time. It explores different areas of the original problem feasible region by dynamically changing the number of super-products used to formulate the master problem. Note that in Phase I, the number of super-products is limited by a threshold value $\K_{max}$ to ensure computational tractability. Phase II has a time limit set to $\timelimitphasetwo$ units of time and it aims to close the optimality gap. To do so, it determines a \textbf{K-EMP} that is equivalent to the original problem, initiates it with the best bounds found in Phase I, and solves it. We first present a clustering-based strategy used to partition the product set prior to the algorithm. We then present the first phase and the second phase of Meta-PBD. The whole algorithm is decribed in Algorithm \ref{alg:global_algo}.

\begin{algorithm}
  \small
  \DontPrintSemicolon
  \KwData{Maximum number of super-products, $\K_{max}$, Time limit on bounds improvement, $\timelimitlb$, Minimal improvement required for the bounds, $\imprlb\%$, Threshold on the number of master solutions explored, $\minmastersols$, Time limit for phase I, $\timelimitphaseone$, Time limit for phase II, $\timelimitphasetwo$}
  1-partition $\leftarrow$ whole product set $\products$ \;
  \For{$\K$ ranging from $1$ to $\K_{max}-1$}{
    Identify largest product subset from the $\K$-partition \;
    Partition it by applying the 2-medoids algorithm, using distance $1-m( p, p')$ for each $( p, p') \in \products$\;
    $\K+1$-partition $\leftarrow$  combine the obtained subsets with all other product subsets from the $\K$-partition unchanged \;
  }
  Launch phase I using the obtained partitions\;
  Launch phase II using the best bounds found in Phase I\;
  \KwResult{Final solution, UB}
  \caption{\label{alg:global_algo}%
    Meta Partial Benders Decomposition}
\end{algorithm}

\subsection{Partitioning the product set}
\label{subsec:partition}
\noindent

In the first phase, the number of super-products used to formulate the master problem can range from 1 to a threshold $\K_{max}$. In our implementation, the partitioning of the product set is based in part on the K-medoids  \cite{Jain:1988,Berkhin:2006} method, a greedy algorithm that partitions $N$ objects into $K$ clusters. Like K-means, K-medoids seeks to put objects into $K$ clusters such that the sum of distances from the objects to the centers of clusters to which they are assigned is minimized. Unlike K-means, in K-medoids the center of each cluster is one of the objects in that cluster.  

K-medoids requires a distance measure between each pair of objects. In the context of Meta-PBD, and due to Theorem \ref{theorem2}, we seek to partition product set into product subsets with high matching rates. As the K-medoids algorithm seeks to minimize total distance, we set a proximity measure between each pair of products $(\product, \product') \in \dot \products ^ 2$ as $\proximity(\product,\product') = 1 - \matching(\product,\product')$. 

For the sake of simplicity, we compute the partitions of the product set that may be used in Phase 1 in a preprocessing fashion. We start with $\K=1$ and we use an incremental approach to determine the different partitions of the product set. At each iteration, we build a $\K+1$-partition of the product set by partitioning the largest product subset of the current $\K$-partition into two product subsets, leaving the other product subsets (and resulting super-products) unchanged. The largest product subset is partitioned  by applying the 2-medoids algorithm. Given Theorem \ref{theorem3}, this strategy ensures that, for $(k_1, k_2) \in \{1,...,\K\}^2$, if $k_2>k_1$ the master problem \textbf{$K_2$-EMP} yields bounds at least as strong as those obtained solving the master problem \textbf{$K_1$-EMP}. This preprocessing phase terminates as a $\K$-partition of the product set for each value of $\K$ ranging from 1 to $\K_{max}$.

\subsection{Phase I}
\label{subsec:phaseI}
\noindent

The first phase iteratively explores different areas of the original problem feasible region and it aims to quickly determine high-quality solutions. At each iteration, it changes the number of super-products and it solves the resulting master problem. It uses solution progress criteria and an assessement of the master computational tractability to determine wether the master should be changed as well as whether the number of super-products should decrease or increase. Based on these decisions, and if the number of super-products can decrase/increase, it applies an integral bisection search to identify a more promising number of super-products. At each iteration, the master problem is solved using the Benders decomposition-based solution algorithm proposed by Belieres et al. \cite{Belieres:2019}. Although that approach is presented in the context of solving a master problem based on a single super-product, it can be easily extended to a master problem based on multiple super-products. 

The first phase is initiated with $\K=1,$ and it starts by solving the \textbf{1-EMP}. The reason behind this choice is that the \textbf{1-EMP} is the most tractable master problem, and thus a good candidate to quickly improve the initial bounds. At each iteration, the algorithm must (i) detect when the progress of the Benders decomposition-based algorithm for the current \textbf{K-EMP} has slowed and identify the reason for this limited progress and (ii) suggest a more adequate number of super-products. We next describe how the algorithm performs each step. 

The algorithm tracks whether the Benders decomposition-based algorithm applied with the current \textbf{K-EMP} was able to improve either the primal or dual bound on the optimal objective function value of the original problem. If $\timelimitlb$ units of time have passed in which neither the primal bound nor the dual bound have improved by at least $\imprlb$, the algorithm counts the number of integer master solutions that have been produced in the last $\timelimitlb$ units of time. If this number is lower than $\minmastersols$, the algorithm judges that the slow progress is due to a lack of computational tractability and it indicates that the number of super-products should decrease. On the other hand, if the number of integer master solutions produced is equal to or greater than $\minmastersols$, the algorithm judges that the current master problem should be strengthened and thus that the number of super-products should increase. As the number of super-products for the next iteration is obtained via an integral bisection search, before stopping the current solution the algorithm first determines if the interval to be bisected is not empty. In that case only, the solution is stopped and the next number of super-products is computed. Figures \ref{fig:bisection_search_1} and \ref{fig:bisection_search_2} illustrate two examples of bisection search with $\K_{max}=5$. 

\begin{figure}[htp]
\centering
\includegraphics[scale=0.66]{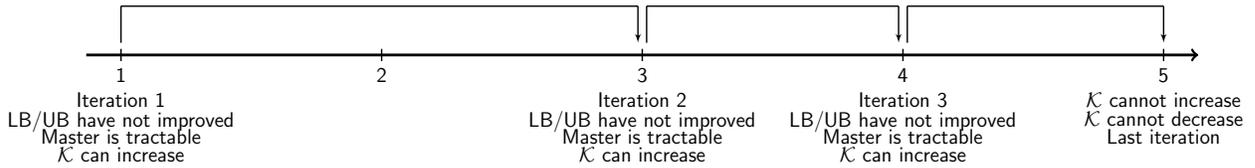}
\caption{First example of bisection search}\label{fig:bisection_search_1}
\end{figure}
\begin{figure}[htp]
\centering 
\includegraphics[scale=0.66]{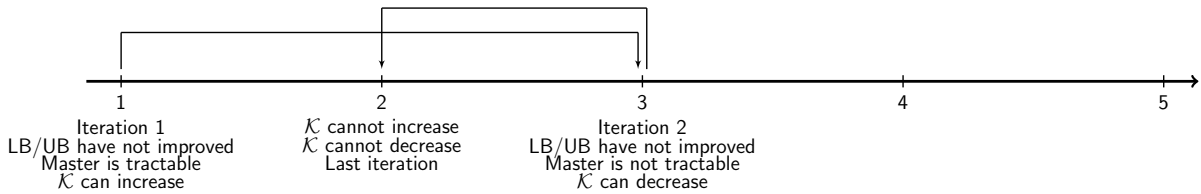}
\caption{Second example of bisection search}\label{fig:bisection_search_2}
\end{figure}

To determine if the number of super-products can decrease or increase, let $K^-$ be the largest number of super-products used previously that is smaller than $\K$ and let $K^+$ be the smallest number of super-products used previously that is larger than $\K$. Then, $\K$ cannot decrease if $\K=1$ or $K^-=\K-1$. Similarly, $\K$ cannot increase if $\K=\K_{max}$ or $K^+=\K+1$. Outside of these cases, the number of super-products to be considered at the next iteration is obtained by an integral bisection search. If the algorithm aims to increase (decrease) the number of super-products, we compute the midpoint between $\K$ and $K^+$ ($K^-$) and we round it up if it is not integer. The obtained value defines the number of super-products to be considered at the next iteration. Note that it may happen that the algorithm seeks to increase $\K$ while $\K<\K_{max}$ and no number of super-products larger than $\K$ was studied previously. In that case, the number of super-products to be considered at the next iteration is set to $\lceil \frac{\K_{max} - \K}{2} \rceil$. 

We note that at each iteration we initiate the current Benders decomposition-based algorithm with the best upper bound and the best lower bound found in previous iterations. In addition, we strengthen the \textbf{K-EMP} at a given iteration with all the Benders cuts generated in  previous iterations. As the Benders cuts \eqref{feasibilityConstraintSMP} and \eqref{optimalityConstraintSMP} only involve $y$ and $z$ variables, they can be re-used from one iteration to the next. Figure \ref{fig:meta-pbd} illustrates the steps taken at each iteration of Meta-PBD. Detailed pseudo-code for the first phase is presented in Algorithm \ref{alg:phase1}, which is in the Appendix (Section \ref{sec:appendix}). 

\begin{figure}[htp]
\centering
\pagestyle{empty}
 

\tikzstyle{decision} = [diamond, draw, fill=white!20, aspect=2,
    text width=6em, text badly centered, node distance=3cm, inner sep=0pt]
\tikzstyle{term} = [rectangle, draw, fill=white!20, 
    text width=5em, text centered, rounded corners, minimum height=4em]
\tikzstyle{proc} = [rectangle, draw, fill=white!20, 
    text width=12em, text badly centered, node distance=3cm, minimum height=4em]
\tikzstyle{proc1} = [rectangle, draw, fill=white!20, 
    text width=6em, text badly centered, node distance=3cm, minimum height=2em]
\tikzstyle{proc2} = [rectangle, draw, fill=white!20, 
    text width=8em, text badly centered, node distance=3cm, minimum height=4em]
\tikzstyle{block} = [fill=white!20, 
    text width=12em, text centered, rounded corners, minimum height=1.5em]
\tikzstyle{coord} = [coordinate, on chain, on grid, node distance=6mm and 25mm]
\tikzstyle{line} = [draw, -latex']
\tikzstyle{cloud} = [draw, ellipse,fill=red!20, node distance=3cm,
    minimum height=2em]

\begin{tikzpicture}[node distance = 2cm, auto]
    \node [proc2, node distance=2cm] (partition) {Select $\mathcal{K}-$partition found in the preliminary phase};
    \node [proc2, right of=partition, node distance=4.5cm] (formulate) {Formulate \textbf{K-EMP} with all Benders cuts and the best bounds from previous iterations};
    \node [proc, right of=formulate, node distance=4.75cm] (solve) {Solve LSNDP via Benders algorithm with \textbf{K-EMP} as master problem until neither primal nor dual bound improve sufficiently fast};
    \node [decision, right of=solve, node distance=5.25cm] (tractable) {Is {K-EMP} tractable?};
    \node [decision, above of=solve, node distance=3.5cm] (increase) {Can $\mathcal{K}$ increase?};
    \node [decision, below of=solve, node distance=3.5cm] (decrease) {Can $\mathcal{K}$ decrease?};
    \node [proc1, left of=partition, node distance=4cm] (bisection) {Bisection search};
    
    
    \path [line] (bisection) -- node [above] {$\mathcal{K}$} (partition);
    \path [line] (partition) -- node [above] {$\Xi$} (formulate);
    \path [line] (formulate) -- (solve);
    \path [line] (solve) -- (tractable);
    \path [line] (tractable) |- node [near start] {Yes} (increase);
    \path [line] (increase) -| node [above] {Yes} (bisection);
    \path [line] (increase) -- node [right] {No} (solve);
    \path [line] (tractable) |- node [near start] {No} (decrease);
    \path [line] (decrease) -| node [below] {Yes} (bisection);
    \path [line] (decrease) -- node [right] {No} (solve);

\end{tikzpicture}
\caption{Iteration of Phase I}\label{fig:meta-pbd}
\end{figure}
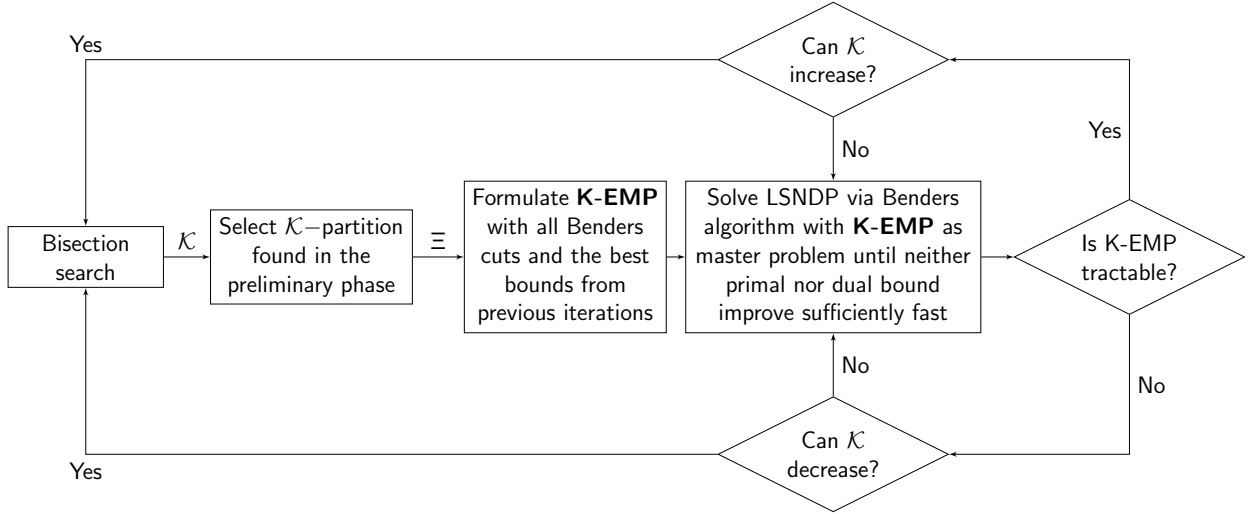

\subsection{Phase II}
\label{subsec:phaseII}
\noindent

In the second phase, we formulate a master problem that is equivalent to the original problem. To determine a partition of the product set that satisfies Theorem \ref{theorem2}, we start with $\K=1$ and we use an incremental approach. At each iteration, we identify a product subset of the current $\K$-partition with a matching rate lower than 1. We then build a $\K+1$-partition of the product set by partitioning the identified product subset into two product subsets, which is done by applying the 2-medoids algorithm. If no product subset is identified, the current $\K$-partition satisfies Theorem \ref{theorem2} and we can formulate the master problem.

The resulting master problem is solved using a generic branch-and-cut implementation rather than the Benders decomposition-based solution algorithm proposed by Belieres et al. \cite{Belieres:2019}. Indeed, since the master problem is equivalent to the original problem, all integer master solutions found in the branch-and-bound tree are feasible, such that using the valid inequalities and the heuristic solutions proposed by Belieres et al. \cite{Belieres:2019} do not bring added value. To reduce the branch-and-bound search tree and accelerate convergence. we initiate the branch-and-cut with the best upper bound and the best lower bound found in the first phase. Detailed pseudo-code for the second phase is presented in Algorithm \ref{alg:phase2}, which is in the Appendix (Section \ref{sec:appendix}).

\section{Computational study}
\label{sec:study}
\noindent

The goal of this computational study is to study the effectiveness of Meta-PBD at solving instances of the LSNDP. We first describe the instances used in this computational study. Specifically, we describe a set of random instances used to analyze the performance of Meta-BPD according to certain parameters, and a set of industrial instances used to assess the performance of our method on realistic cases. Second, we describe how this computational study was conducted. Next, we analyze how the number of super-products impacts the bounds produced by the master problem as well as the time required to solve the master problem. Finally, we study the performance of Meta-PBD on both the random instances and the industrial instances.

\subsection{Instances}
\label{subsec:instances}
\noindent

We first describe the set of random instances generated for this computational study. Then, we briefly present the set of industrial instances taken from \cite{Belieres:2020}.

\subsubsection{Random instances}
\label{subsubsec:random_instances}
\noindent

These instances are inspired by the operations of our industrial partner and produced by a random generator similar to the one described in \cite{Belieres:2019}. This random generator produces a static graph $\flatnetwork$ and a time dimension based on the following parameters: the size of the node-set $\flatnodes$, the connectivity radius $\alpha$, the number of days in the planning horizon $D$, and the number of time-periods per day $\Delta$. The connectivity radius is a threshold on how close two nodes must be to each other for transportation to be an option. See \cite{Belieres:2019} for details regarding how the network $\flatnetwork$ and the time dimension are constructed. Likewise, customer requests, vehicle capacities, and warehouse capacities are defined in a similar way to what is done in \cite{Belieres:2019}.

Our instances differ from those presented in \cite{Belieres:2019} as we generate products and assign products to suppliers in a different fashion. Namely, \cite{Belieres:2019} did not explicitly recognize the concept of product families. Thus, the generator used in this study accepts two more parameters: (i) the cardinality of the product families set $\families$, and, (ii) a supply probability $\productpercentage$. The generator randomly assigns each product to one product family and one, two, or three product families to each supplier. Then the generator randomly distributes products to the suppliers. Specifically, each supplier that manufactures a given product family has a probability $\productpercentage$ to offer each product in that product family. 
  
The random instances are based on combinations of the following parameter values: $|\flatnodes| = \{50\}$, $\alpha = \{10, 30\}$, $D = 30$ days, $\Delta = \{2, 3, 4\}$, $\families = \{7\}$, $|\products| = \{100\}$, and $\productpercentage = \{25\%, 50\%, 75\%\}$. In addition, due to the randomness of the instance generator, we generated 5 instances for each of the 18 possible combinations of parameter values, yielding 90 instances in total.

Recall that for a pair of products ($p_{i}, p_{j}),$ the matching rate $m(p_{i},p_{j})$ measures the fraction of suppliers that manufacture both products. Thus, matching rates are directly impacted by the value of the parameter $\productpercentage.$ As a result, for each generated instance we compute the matching rate of each product family, which indicates whether or not products of that family should be aggregated. We report in Table \ref{tbl:matching_per_pp}  the average product family matching rates, averaged over instances generated with the same value of $\productpercentage$. Not surprisingly, the larger the value of $\productpercentage$ the larger the resulting matching rate.

\begin{table}[htp]
\centering
\caption{Average matching rate of the product families}\label{tbl:matching_per_pp}
\begin{adjustbox}{max width=\textwidth}
\begin{tabular}{lrrr}
\hline
$\productpercentage$ & $25\%$ & $50\%$ & $75\%$ \\
\hline
Matching rate & 46.51\% & 57.42\% &  76.77\% \\
\hline
\end{tabular}
\end{adjustbox}
\end{table}

\subsubsection{Industrial instances}
\label{subsubsec:industrial_instances}
\noindent

These instances are taken from the article \cite{Belieres:2020} and they reflect the operations performed by a 3PL in the supply chain management of a French restaurant chain. These instances are constructed in part from real data provided by our industrial partner. These data include: geographic positions of the stakeholders, travel times between each pair of stakeholders, transportation costs, storage costs, product families manufactured by each supplier, customers demand schedules and product demand seasonalities. We refer the interested reader to \cite{Belieres:2020} for more details on the real data and/or the generation of the industrial instances. 

In this study, we focus on the second set of industrial instances described in \cite{Belieres:2020}. These instances are referred to as \textit{difficult} instances as they involve time-expanded networks of significantly larger scale than those involved in the first set of industrial instances. These instances are based on combinations of the following parameter values: $|\flatnodes| = \{67\}$, $\alpha = \{0, 20, 40, 60\}$, $D = 14$ days, $\Delta = \{2, 4, 6\}$, $\families = \{3\}$, $|\products| = \{131\}$, and 2 demand seasonalities (summer and winter).

\subsection{Setup of study}
\label{subsec:setup}
\noindent

To assess the effectiveness of \textbf{Meta-PBD}, we compare its performance against multiple benchmarks. For the first, \textbf{CPLEX}, the LSNDP is solved with CPLEX branch-and-cut solver with all parameter values left at their defaults. The next three benchmarks are direct applications of the Benders decomposition-based algorithm proposed in \cite{Belieres:2019} on different partial decompositions of the LSNDP. These approaches are referred to as \textit{static} PBD-approaches since the considered master problem solved remains unchanged through the optimization process. In the method \textbf{Single}, the master problem is formulated using a single super-product that aggregates all the products in the instance. On the other hand, the next two methods involve master problems based on multiple super-products. The method \textbf{Families} partitions the set of products based on the product families to which they belong. Thus, the resulting master problem is formulated with one super-product for each product family. The method \textbf{Random} randomly partitions the set of products into $\vert \mathcal{F} \vert$ subsets, and thus randomly determines how super-products and the resulting master problem are formulated. 

Each method is seeded with the same heuristic solution $(x_h, y_h)$, which is obtained by solving the linear relaxation of the LSNDP and then rounding up the value of each vehicle variable $y_{ij}^{tt'}$ in its optimal solution. All algorithms are coded in C++ and executed on an Intel Xeon E5-2695 processor with 16 GB of memory under Linux 16.04. Linear and integer programs were solved using Cplex 12.7. All algorithms are executed with a stopping criteria of a proven optimality gap of 1\%. Regarding parameters for Meta-PBD, we set $\K_{max}$ to 10, $\timelimitlb$ to $1/18$ of the global time limit, $\imprlb\%$ to $1\%$, $\minmastersols$ to 1, $\timelimitphaseone$ to $1/3$ of the global time limit, and $\timelimitphasetwo$ to $2/3$ of the global time limit. Note that these values were determined through a set of tuning experiments.

\subsection{Impact of $\K$ on \textbf{K-EMP}}
\label{subsec:impact_of_k}
\noindent

We first study the impact that the number of super-products has on the resolution of the resulting master problem, both in terms of the bound produced and of the computational time needed. Note that we use the random instances, for which 7 product families are involved.

For each instance and all values of $\K$ from 1 to 7, we apply the K-medoids algorithm to partition the set of products $\products$. Based on this partition, we formulate the resulting \textbf{K-EMP} and we solve its linear programming relaxation. In addition, for each instance we solve the linear programming relaxation of the LSNDP. For each linear programming relaxation (LPR) we consider two statistics: (i) the objective function value of the optimal solution, $r^{sol}_{LPR},$ and (ii) the computational time required for its solution, $r^{time}_{LPR}$.

To measure the impact that $\K$ has on the quality of the bounds produced by the master problem as well as on the time required for solving the master problem, for each value of $\K$ and for each instance we compute the following indicators:
$$ \text{lb-root-gap} = \frac{r^{sol}_{LSNDP} - r^{sol}_{K-EMP}}{r^{sol}_{LSNDP}} \hspace{1cm} \text{root-time-ratio} = \frac{r^{time}_{K-EMP}}{r^{time}_{LSNDP}}$$
lb-root-gap indicates the gap between the optimal solution of the \textbf{K-EMP} linear programming relaxation and the optimal solution of the LSNDP linear programming relaxation. root-time-ratio indicates the ratio between the computational time required for solving the \textbf{K-EMP} linear programming relaxation and the computational time required for solving the LSNDP linear programming relaxation. In Figure \ref{fig:root}, we display for different values of $\K$ the average lb-root-gap and root-time-ratio, averaged over all instances.

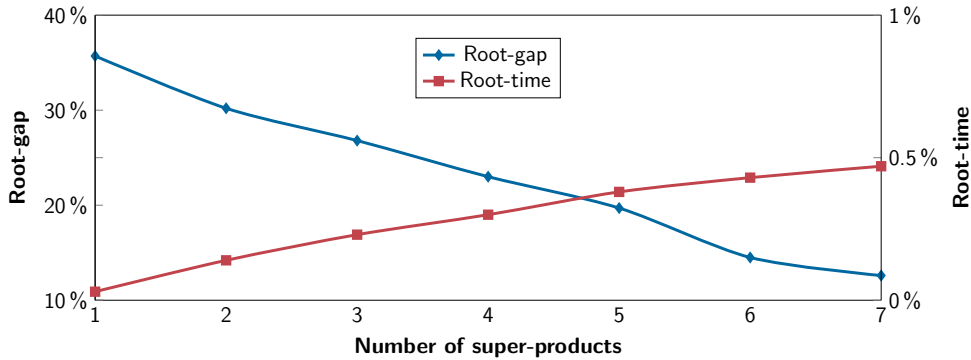
\begin{figure}[htp]
\centering
\begin{tikzpicture}
  \centering
  \begin{axis}[
        title={},
        width  = 1*\textwidth,
        height = 7cm,
        axis y line*=left,
        major x tick style = transparent,
         nodes near coords align={vertical},
   		 yticklabel=\pgfmathprintnumber{\tick}\,$\%$,
        ymin=10, ymax=40,
        xmin=1, xmax=7,
    	xlabel={\textbf{Number of super-products}},
    	ylabel=\textbf{Root-gap},
    	ytick={10,20,30,40},
    	xtick={1,2,3,4,5,6,7},
    ]
\addplot[smooth,mark=diamond*,ultra thick,color=colorRoot2] plot coordinates {
      (1,35.7)  (2,30.2) (3,26.8) (4,23.0) (5,19.7) (6,14.5) (7,12.6)
    }; 
\end{axis}

  \begin{axis}[
        title={},
        width  = 1*\textwidth,
        height = 7cm,
        ylabel near ticks, yticklabel pos=right,
        axis x line=none,
        nodes near coords align={vertical},
   		yticklabel=\pgfmathprintnumber{\tick}\,$\%$,
        ymin=0, ymax=1,
        xmin=1, xmax=7,
    	ylabel=\textbf{Root-time},
        ytick={0,0.5,1},
    	xtick=data,
    	legend style={at={(0.5,0.92)},anchor=north},
    ]
\addlegendimage{smooth,mark=diamond*,ultra thick,color=colorRoot2}\addlegendentry{Root-gap}
\addplot[smooth,mark=square*,ultra thick,color=colorTime] plot coordinates {
        (1,0.03) (2,0.14) (3,0.23) (4,0.30) (5,0.38) (6,0.43) (7,0.47)
    }; \label{time} \addlegendentry{Root-time}
\end{axis}
\end{tikzpicture}
\caption{Root-gaps and computational time ratios for the \textbf{K-EMP}}\label{fig:root}
\end{figure}

We see that the \textbf{K-EMP} based on a single super-product gives a root-gap of $35.7\%$. However the average \textit{root-gap} decreases as the number of super-products used to formulate the master increases. This confirms that considering more super-products in the master leads to better approximate the original problem. This improvement is significant, as the average \textit{root-gap} decreases by more than $20\%$ when  7 super-products are used to formulate \textbf{K-EMP}. As expected, the computational time for solving the master problem linear programming relaxation increases with the number of super-products. However, in the worst case, the linear programming relaxation of \textbf{K-EMP} is solved in less than $0.5\%$ of the time needed to solve the LSNDP root relaxation.

Belieres et al. \cite{Belieres:2019} propose three classes of valid inequalities for strengthening a \textbf{K-EMP} based on a single super-product. Since the same valid inequalities are used in this study, we ran a similar experiment to the one just described, albeit with these constraints added to the formulation of each \textbf{K-EMP}. We then recalculated the \textit{root-gap} and \textit{root-time} statistics in the same manner. We observed the same trends in those statistics as in Figure \ref{fig:root}. Namely, the larger the value of $\K,$ the smaller the value of \textit{root-gap} and the larger the value of \textit{root-time}. However, the magnitudes were much different, as the \textit{root-gap} equals $0.63\%$ for $K=1$ and $0.38\%$ for $K=7$. We thus observe that for $K=1$, the optimal solution of the \textbf{K-EMP} linear relaxation already provides a very tight bound on the LSNDP linear relaxation as the valid inequalities are added.

\subsection{Results obtained on the random instances}
\label{subsec:results_random}
\noindent

We analyze the results obtained on the random instances with a time limit of 3 hours. We first compare our approach to various benchmark methods and we computationally demonstrate that it strictly outperforms the algorithm proposed by Belieres et al. \cite{Belieres:2019}, i.e. \textbf{Single}. We then analyze the performance of our approach in more detail.

\subsubsection{Benchmarking performance of Meta-PBD}
\label{subsubsec:performance_random}
\noindent
We compare the performance of \textbf{Meta-PBD} variants with that of the following benchmark methods: \textbf{CPLEX}, \textbf{CPLEX-Benders}, \textbf{Single}, \textbf{Families} and \textbf{Random}. We do so based in part on the objective function value, $\text{UB}_{x}$, of the best primal solution produced by method x, as well as the strongest dual bound, $\text{LB}_{x}$, it produced. Let $\text{UB}_{Best}$ denote the best objective function value over all methods. Similarly, we let $\text{LB}_{Best}$ denote the best dual bound produced by all methods. For each instance and each method x, we compute two performance indicators:

$$ \text{gap}_{UB} = \frac{\text{UB}_{x} - \text{UB}_{Best}}{\text{UB}_{x}} \hspace{1cm} \text{gap}_{LB} = \frac{\text{LB}_{Best} - \text{LB}_{x}}{\text{LB}_{Best}}$$

Both indicators are non-negative and equal 0 if method $x$ produced the primal/dual solution with best objective function value over all methods.  We also compute the average optimality gap at termination for method $x$ over all instances as $(\text{UB}_{x}.- \text{LB}_{x})/\text{UB}_{x}$. 

In Table \ref{tbl:average} we present averages of the optimality gap and each performance indicator for each method. We also report the number of instances solved as well as $\text{nb.LB}_{Best}$ and $\text{nb.UB}_{Best}$, i.e. the number of instances for which method $x$ found the best lower/upper bound over all methods. Note that for one instance, the best upper bound was found by both \textbf{Meta-PBD} and \textbf{Single}. Best values are noted in bold.

\begin{table}[htp]
\centering
\caption{Comparing \textbf{META-PBD} to the benchmarks}\label{tbl:average}
 \begin{tabular}{lrrrrrr} \hline
 Method & Optimality Gap & Solved & $\text{gap}_{LB}$ & $\text{nb.LB}_{Best}$ & $\text{gap}_{UB} $ & $\text{nb.UB}_{Best}$   \\
\hline
 \textbf{CPLEX} & 10.47\% & 2 & 0.06\% & 1 & 7.17 & 2 \\
 \textbf{Families} & 5.33\% & 0 & 0.50\% & 0 & 2.09\% & 16 \\
 \textbf{Meta-PBD} & \textbf{3.06\%} & \textbf{11} & \textbf{0.00\%} & \textbf{88} & \textbf{0.23\%} & \textbf{65} \\
 \textbf{Single} & 5.02\% & 0 & 0.66\% & 1 & 1.60\% & 7 \\
 \textbf{Random} & 8.45\% & 0 & 0.64\% & 0 & 5.20\% & 1 \\  \hline
    \end{tabular}
 \end{table}

\textbf{Meta-PBD} significantly outperforms all the benchmarks regardless of the indicator considered. In particular, we observe that \textbf{Meta-PBD} provides primal solutions much better than those produced by the other benchmarks. It is also noteworthy that \textbf{Meta-PBD} provides strict improvements over the algorithm proposed by Belieres et al. \cite{Belieres:2019}, i.e. \textbf{Single}. Overall, our approach provides a gap at termination that is more than 2 percent better that of \textbf{Single}. Similarly, we observe that \textbf{Meta-PBD} solves 11 instances while \textbf{Single} never closes the optimality gap within $1\%.$ The only benchmark that solves some instances to optimality is \textbf{CPLEX}. Nevertheless, the performance of \textbf{CPLEX} is quite poor, as it yields the largest optimality gap on average. Although \textbf{CPLEX} produces strong dual bounds, it encounters difficulties on the primal side. Ultimately, all the Partial Benders Decomposition-based strategies provide better primal solutions than \textbf{CPLEX} overall.

We next turn our attention to the Partial Benders Decomposition-based strategies, i.e. \textbf{Meta-PBD} and the static PBD-approaches \textbf{Single}, \textbf{Families}, and \textbf{Random}. We first observe that all strategies produce strong dual bounds, and that the quality of these bounds does not vary significantly from one strategy to another. This result is highly anticipated since we observed in subsection \ref{subsec:impact_of_k} that, when valid inequalities are added, the linear relaxation of the master problem with a single super-product already provides a very tight bound on the LSNDP linear relaxation. As a result, the improvement of the bounds produced by the master problem is of small magnitude as the number of super-products increases. On the other hand, we observe greater differences in the quality of the primal solutions produced.

As \textbf{Random} and \textbf{Families} both involve 7 super-products, comparing these approaches in terms of primal solutions provides evidence for our hypothesis that how the product set is partitioned is critical for the resulting master problem to enable a Benders decomposition-based algorithm to perform well. The performance of \textbf{Single} is very close to that of \textbf{Families}. Overall, \textbf{Single} yields a better optimality gap and produces better primal solutions on average. Nevertheless, \textbf{Families} finds the best primal solution over all methods much more often than \textbf{Single}. This result suggests that the number of super-product to consider in a static PBD-approach depends on the nature of the instance solved. To confirm this hypothesis, we compare the performance \textbf{Single} and \textbf{Families} for different values of $\productpercentage$.

As noted, we generated instances for three different values of $\productpercentage$, the probability a supplier produces a product in a family it supplies. In addition, the larger the value of $\productpercentage$, the more likely the condition to which Theorem \ref{theorem2} holds, in which case the master problem based on product families is equivalent to the original problem. For each value of $\productpercentage$, we report in Table \ref{tbl:compare} the number of instances for which \textbf{Single} provides a better primal/dual solution than \textbf{Families}, and vice versa. Best values are noted in bold.

\begin{table}[htp]
\centering
\caption{Comparing the primal/dual solutions produced by \textbf{Families} and \textbf{Single}}\label{tbl:compare}
\begin{adjustbox}{max width=\textwidth}
\begin{tabular}{l|rr|rr}
\hline
&\multicolumn{2}{c}{Best primal solution} & \multicolumn{2}{c}{Best dual solution} \\
$\productpercentage$ & \textbf{Families} & \textbf{Single} & \textbf{Families} & \textbf{Single}\\
\hline
25\% & \textbf{18} & 12 & 13 & \textbf{17} \\
50\% & \textbf{26} & 4 & \textbf{16} & 14 \\
75\% & \textbf{30} & 0 & \textbf{20} & 10 \\
\hline
\end{tabular}
\end{adjustbox}
\end{table}

As anticipated, \textbf{Families} tends to perform better than \textbf{Single} as the value of $\productpercentage$ increases. This confirms that, in the case of a static PBD-approach, the number of super-products to consider depends greatly on the nature of the instance. In the context of multi-product supply chains, products are often already classified into different families, and such classifications should be used to determine the partition of the product set. Nevertheless, for a given supply chain there exist in practice multiple classifications of the products with different levels of granularity. For example, the family of fresh products could be decomposed into fruit products, vegetable products, and meat products. As a result, it is assumed that using different partitions of the product set with different levels of granularities can be beneficial, which motivates our Meta-PBD framework.

\subsubsection{Analyzing the performance of Meta-PBD}
\label{subsubsec:analyze_random}
\noindent

As observed in Table \ref{tbl:average}, \textbf{Meta-PBD} produces primal solutions significantly stronger than those computed by the static PBD-approaches. Thus, we investigate to understand why \textbf{Meta-PBD} outperforms the benchmarks. 

We assess the interest of changing dynamically the number of super-products in the first phase of \textbf{Meta-PBD}. To do so, we study the performance study of each iteration of \textbf{Meta-PBD}'s first phase. We report in Table \ref{tbl:iteration_distrib} the distribution of the number of iterations performed during the first phase of \textbf{Meta-PBD}. For example, column "3 iter." indicates the percentage of instances for which the Benders decomposition-based algorithm was executed three times during the first phase of \textbf{Meta-PBD}. 
We report in Table \ref{tbl:results_per_decomp} the improvement obtained in the primal bound and the dual bound during each iteration of \textbf{Meta-PBD}'s first phase. We note that since \textbf{Meta-PBD}'s first phase did not execute 5 iterations for all instances, for a given number of iterations, we average over instances wherein \textbf{Meta-PBD}'s first phase executed as many iterations.

\begin{table}[htp]
	\centering
	\caption{Distribution of instances by \# of iterations performed during the first phase of \textbf{Meta-PBD}} \label{tbl:iteration_distrib}
	\begin{tabular}{ccccc}
	\hline
 1 iter. & 2 iter. & 3 iter. & 4 iter. & 5 iter.  \\ \hline
 32.22\% & 34.44\% & 21.11\% & 11.11\% & 1.11\% \\
\hline
	\end{tabular}
\end{table}


\begin{table}[htp]
	\centering
	\caption{Performance per iteration of the first phase of \textbf{Meta-PBD}} \label{tbl:results_per_decomp}
	\begin{tabular}{ccc}
	\hline
        Iter.& LB Impr. & UB Impr. \\ \hline
		1 & 0.44\% & 8.54\% \\
		2 & 0.15\% & 3.52\% \\
		3 & 0.05\% & 1.30\% \\
		4 & 0.03\% & 0.06\% \\
		5 & 0.00\% & 0.00\% \\ \hline
	\end{tabular}
\end{table}

Most of the improvement on the primal and the dual bounds is performed during the first iteration. Nevertheless, while the improvement decrease with the number of iterations, we observe that the primal bound is significantly improved in both the second and the third iteration. For $80\%$ of the instances, the best primal solution obtained at the end of \textbf{Meta-PBD}'s first phase was computed during the last iteration.

We now compare the primal solutions obtained at the end of the first phase, i.e. \textbf{Meta-PBD-1h}, to those computed with the static PBD-approaches in the same amount of time, i.e. \textbf{Single-1h} and \textbf{Families-1h}. For each instance, we measure a primal gap between the solution computed by \textbf{Meta-PBD-1h} and those computed by \textbf{Single-1h} and \textbf{Families-1h}. A positive gap indicates that \textbf{Meta-PBD-1h} found a better primal solution than the considered benchmark. Overall, we observe that \textbf{Meta-PBD-1h} computes primal solutions $4.22\%$ better than \textbf{Families-1h} and $2.87\%$ better than \textbf{Single-1h}. This clearly demonstrates the added value of changing dynamically the number of super-products used to formulate the master problem.

A critical factor in the performance of a Benders decomposition-based algorithm is its ability to quickly solve master problems, which in turn yield subproblems that may be solved as part of a process for creating feasible solutions to the original problem. Thus, we report in Table \ref{tbl:subproblems} the average number of subproblems generated by \textbf{Meta-PBD-1h}, \textbf{Single-1h} and \textbf{Families-1h}, as well as the average percentage of those subproblems that are feasible. 

\begin{table}[htp]
	\centering
	\caption{Subproblem generation and feasibility after 1 hour of computation} \label{tbl:subproblems}
	\begin{tabular}{lrrr}
	\hline
& \textbf{Families-1h} & \textbf{Meta-PBD-1h} & \textbf{Single-1h} \\
	\hline
Average number of subproblems generated & 3.72 & 5.16 & 5.76 \\
\% subproblems that are feasible & 47.76\% & 46.55\% & 33.01\% \\
\% subproblems that are infeasible & 52.24\% & 53.45\% & 66.99\% \\
\hline
	\end{tabular}
\end{table}

This table demonstrates how the number of super-products used to construct the \textbf{K-EMP} impacts the number and quality of master problem solutions generated. As \textbf{Single-1h} solves the most computationally tractable master problem, it yields the greatest number of subproblems. However, as that master problem has the poorest approximation of the original problem, those subproblems are the least likely to be feasible amongst the three methods. Conversely, with a master problem that is stronger, but harder to solve, \textbf{Families-1h} generates the fewest number of subproblems. However, those subproblems are the most likely to be feasible amongst the three methods. \textbf{Meta-PBD-1h} achieves a balance between the other two methods. It generates nearly as many subproblems as \textbf{Single-1h}, but the percentage of subproblems that are feasible is nearly to that of \textbf{Families-1h}. 

We now illustrate in Figure \ref{fig:ub} the distribution of occurrences of primal solution improvement over the first hour of computation. The computational time of 3,600 seconds is divided into 10 intervals of 360 seconds. Thus, a value of 10 for the first interval indicates that the incumbent primal solution was improved 10 times within the 360 first seconds of computation. 

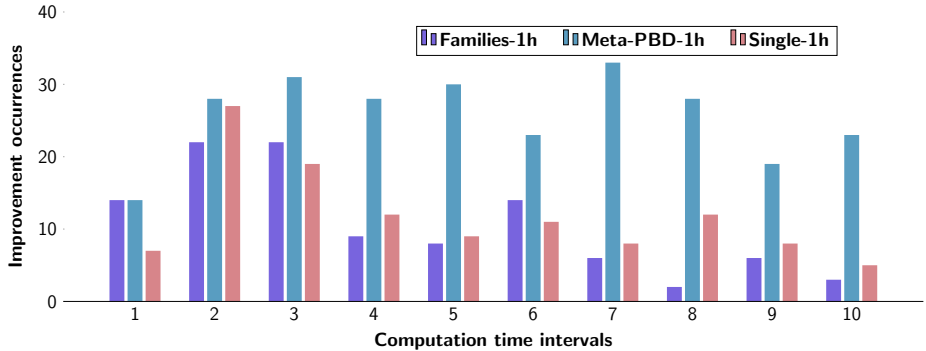
\begin{figure}[htp]
\centering
\begin{tikzpicture}
  \centering
  \begin{axis}[
        ybar, axis on top,
        width  = 1.25*\textwidth,
        height = 8cm,
        bar width=0.33cm,
        major grid style={draw=white},
        enlarge y limits={value=.,upper},
        axis x line*=bottom,
        axis y line*=left,
        y axis line style={opacity=0},
        tickwidth=0pt,
        enlarge x limits=true,
        ymin=0, ymax=40,
       legend style={
            at={(0.9,0.9)},
            anchor=east,
            legend columns=-1,
            /tikz/every even column/.append style={column sep=0.5cm}
        },
        ylabel={\textbf{Improvement occurrences}},
        xlabel={\textbf{Computation time intervals}},
        symbolic x coords={
          1,2,3,4,5,6,7,8,9,10},
       xtick=data,
    ]
    
   \addplot [draw=none,fill=colorRoot1!65] coordinates {(1,14)(2,22)(3,22)(4,9)(5,8)(6,14)(7,6)(8,2)(9,6)(10,3)};
   \addplot [draw=none,fill=colorRoot2!65] coordinates {(1,14)(2,28)(3,31)(4,28)(5,30)(6,23)(7,33)(8,28)(9,19)(10,23)}; 
   \addplot [draw=none,fill=colorTime!65] coordinates {(1,7)(2,27)(3,19)(4,12)(5,9)(6,11)(7,8)(8,12)(9,8)(10,5)}; 
      \legend{\textbf{Families-1h},\textbf{Meta-PBD-1h},\textbf{Single-1h}}
  \end{axis}
\end{tikzpicture}
\caption{Distribution of the occurrences of primal solution improvement after 1 hour of computation}\label{fig:ub}
\end{figure}

In Figure \ref{fig:ub}, we see that \textbf{Meta-PBD-1h} improves the incumbent throughout the optimization process. This is in contrast to \textbf{Single-1h} and \textbf{Families-1h}, both of which find the majority of the improved primal solutions early in their execution. Overall, \textbf{Meta-PBD-1h} improves the incumbent more often than \textbf{Single-1h} and \textbf{Families-1h} combined. Thus, we conclude that dynamically varying the subproblem information used to formulate the master problem enables to change the portion of the feasible region that is explored and allows to discover more promising solutions.

We now demonstrate the interest of \textbf{Meta-PBD}'s second phase, which formulates and solves a master problem equivalent to the original problem. To do so, we compare the results obtained with \textbf{Meta-PBD} to those obtained when performing the first phase alone with the same time limit, i.e. \textbf{Phase-I-3h}, or the seconde phase alone with the same time limit, i.e. \textbf{Phase-II-3h}. In Table \ref{tbl:compare_phases}, we indicate values for the same performance indicators as those reported in Table \ref{tbl:average}. Note that here, we use the best primal/dual bound found over \textbf{Meta-PBD}, \textbf{Phase-I-3h}, and \textbf{Phase-II-3h} to compute values for the performance indicators.

\begin{table}[htp]
\centering
\caption{Comparing \textbf{META-PBD} to \textbf{Phase-I-3h} and \textbf{Phase-II-3h}}\label{tbl:compare_phases}
 \begin{tabular}{lrrrrrr} \hline
 Method & Optimality Gap & Solved & $\text{gap}_{LB}$ & $\text{nb.LB}_{Best}$ & $\text{gap}_{UB} $ & $\text{nb.UB}_{Best}$   \\
\hline
 \textbf{Meta-PBD} & \textbf{3.06\%} & \textbf{11} & \textbf{0.00\%} & \textbf{62} & \textbf{0.40\%} & \textbf{53} \\
 \textbf{Phase-I-3h} & 3.71\% & 0 & 0.50\% & 1 & 0.58\% & 43 \\
 \textbf{Phase-II-3h} & 8.72\% & 8 & 0.01\% & 27 & 6.26\% & 6 \\  \hline
    \end{tabular}
 \end{table}
 
\textbf{Phase-I-3h} provides a significantly better average optimality gap at termination than \textbf{Phase-II-3h}, which can be attributed to the difference in the primal solutions computed. This result indicates that directly solving a master problem strengthened with a large amount of subproblem information is not an efficient strategy, as it yields a branch-and-bound search tree that is prohibitively large. On the other hand, \textbf{Phase-II-3h} manages to solve 8 instances to optimality, against 0 instances for \textbf{Phase-I-3h}. This demonstrates that, while strengthening the master problem with aggregated subproblem information allows to quickly compute high-quality primal solutions, having a loss of information compared to the original problem prevents to compute the optimal master solution. Our approach strikes a compromise as it significantly outperforms both benchmarks. Ultimately, \textbf{Meta-PBD} solves more instances than \textbf{Phase-II-3h} and it provides better primal solutions than \textbf{Phase-I-3h}. When executed after the first phase, the second phase improves the best primal bound by $1.06\%$ on average. Similarly, it improves the best dual bound by $0.59\%$ on average.

\subsection{Results obtained on the industrial instances}
\label{subsec:results_industrial}
\noindent

In this section, we assess how \textbf{Meta-PBD} compares with the best performing benchmark, \textbf{Single}, on the set of industrial instances with a maximum run-time of 3 hours. The instances, which are described by the name ``LSNDP\textunderscore{}$\alpha$\textunderscore{}$\Delta$\textunderscore{}S,'' vary according to three parameters: the connectivity radius ($\alpha$), the number of time-periods per day ($\Delta$), and the season considered (S). We set a maximum run-time of 5 hours for each run. In Table \ref{tbl:results_industrial}, we report the optimality gap, the dual solution, and the primal solution computed by \textbf{Meta-PBD} and \textbf{Single}. The best values are noted in bold. Note that none of the instances could be solved by either method within the time limit.

 \begin{table}[htp]
\begin{minipage}{0.49\textwidth}
\centering
\begin{adjustbox}{max width=0.9\textwidth}
    \begin{tabular}{l|rrr|rrr}  \hline
    &\multicolumn{3}{c}{\textbf{Meta-PBD}} & \multicolumn{3}{c}{\textbf{Single}} \\
    Instance & Gap & LB & UB & Gap & LB & UB \\ \hline
         LSNDP\textunderscore{}0\textunderscore{}2\textunderscore{}S & 2.20. & 112,069. & 114,585. & \textbf{2.00}. & \textbf{112,212}. & \textbf{114,506} \\
        LSNDP\textunderscore{}0\textunderscore{}2\textunderscore{}W & 2.29. & 94,867. & 97,090. & 2.45. & \textbf{94,902}. & 97,284 \\
        LSNDP\textunderscore{}0\textunderscore{}4\textunderscore{}S & 7.74. & 107,709. & 116,740. & \textbf{5.50}. & \textbf{107,857}. & \textbf{114,130} \\
        LSNDP\textunderscore{}0\textunderscore{}4\textunderscore{}W & 4.04. & 91,320. & 95,163. & \textbf{3.77}. & \textbf{91,342}. & \textbf{94,925} \\
        LSNDP\textunderscore{}0\textunderscore{}6\textunderscore{}S & 8.27. & 106,210. & 115,780. & 9.02. & \textbf{107,245}. & 117,884 \\
        LSNDP\textunderscore{}0\textunderscore{}6\textunderscore{}W & 10.85. & 89,941. & 100,887. & \textbf{9.86}. & \textbf{90,913}. & \textbf{100,855} \\
        LSNDP\textunderscore{}20\textunderscore{}2\textunderscore{}S & \textbf{4.56}. & \textbf{104,291}. & \textbf{109,270}. & 18.34. & 97,984. & 119,985 \\
        LSNDP\textunderscore{}20\textunderscore{}2\textunderscore{}W & \textbf{4.11}. & \textbf{88,562}. & \textbf{92,354}. & 20.60. & 81,652. & 102,832 \\
        LSNDP\textunderscore{}20\textunderscore{}4\textunderscore{}S & \textbf{7.53}. & \textbf{101,169}. & \textbf{109,413}. & 35.18. & 92,411. & 142,560 \\
        LSNDP\textunderscore{}20\textunderscore{}4\textunderscore{}W & \textbf{6.40}. & \textbf{85,464}. & \textbf{91,310}. & 23.75. & 76,590. & 100,441 \\
        LSNDP\textunderscore{}20\textunderscore{}6\textunderscore{}S & \textbf{27.17}. & \textbf{100,771}. & 138,356. & 32.58. & 87,284. & \textbf{129,466} \\
        LSNDP\textunderscore{}20\textunderscore{}6\textunderscore{}W & \textbf{7.79}. & \textbf{85,433}. & \textbf{92,655}. & 39.10. & 74,748. & 122,730 \\ \hline
           \end{tabular}
   \end{adjustbox}
 \end{minipage}
\begin{minipage}{0.49\textwidth}
\centering
\begin{adjustbox}{max width=0.9\textwidth}
    \begin{tabular}{l|rrr|rrr}   \hline
    &\multicolumn{3}{c}{\textbf{Meta-PBD}} & \multicolumn{3}{c}{\textbf{Single}} \\
    Instance & Gap & LB & UB & Gap & LB & UB \\ \hline
        LSNDP\textunderscore{}40\textunderscore{}2\textunderscore{}S & \textbf{4.73} & \textbf{101,917} & \textbf{106,976} & 52,52 & 83,237 & 175,318 \\
        LSNDP\textunderscore{}40\textunderscore{}2\textunderscore{}W & \textbf{4.28} & \textbf{86,585} & \textbf{90,460} & 30,94 & 71,663 & 103,774 \\
        LSNDP\textunderscore{}40\textunderscore{}4\textunderscore{}S & \textbf{8.36} & \textbf{98,095} & \textbf{107,046} & 67,49 & 81,312 & 250,086 \\
        LSNDP\textunderscore{}40\textunderscore{}4\textunderscore{}W & \textbf{8.46} & \textbf{83,392} & \textbf{91,095} & 58,37 & 65,025 & 156,210 \\
        LSNDP\textunderscore{}40\textunderscore{}6\textunderscore{}S & \textbf{10.20} & \textbf{97,708} & \textbf{108,812} & 56,62 & 72,382 & 166,844 \\
        LSNDP\textunderscore{}40\textunderscore{}6\textunderscore{}W & \textbf{10.43} & \textbf{82,926} & \textbf{92,579} & 59,47 & 62,174 & 153,406 \\
        LSNDP\textunderscore{}60\textunderscore{}2\textunderscore{}S & \textbf{4.59} & \textbf{101,338} & \textbf{106,209} & 36,04 & 81,467 & 127,375 \\
        LSNDP\textunderscore{}60\textunderscore{}2\textunderscore{}W & \textbf{4.34} & \textbf{86,032} & \textbf{89,932} & 33,03 & 70,320 & 105,004 \\
        LSNDP\textunderscore{}60\textunderscore{}4\textunderscore{}S & \textbf{5.68} & \textbf{98,128} & \textbf{104,034} & 41,87 & 77,019 & 132,501 \\
        LSNDP\textunderscore{}60\textunderscore{}4\textunderscore{}W & \textbf{7.44} & \textbf{83,083} & \textbf{89,761} & 50,64 & 61,115 & 123,811 \\
        LSNDP\textunderscore{}60\textunderscore{}6\textunderscore{}S & \textbf{25.31} & \textbf{97,717} & \textbf{130,822} & 76,97 & 71,333 & 309,785 \\
        LSNDP\textunderscore{}60\textunderscore{}6\textunderscore{}W & \textbf{8.17} & \textbf{82,982} & \textbf{90,362} & 54,71 & 60,461 & 133,509 \\ \hline
        \end{tabular}
   \end{adjustbox}
 \end{minipage}
 \caption{Comparing \textbf{Meta-PBD} to \textbf{Single} on the industrial instances}\label{tbl:results_industrial}
 \end{table}

As for the results provided in Table \ref{tbl:average}, one first observes that \textbf{Meta-PBD} strictly outperforms \textbf{Single}, especially as the size of the instance increases. Overall, \textbf{Meta-PBD} produces a better primal solution than \textbf{Single} for 19 out of the 24 instances, and a better dual solution than \textbf{Single} for 18 out of the 24 instances. The average optimality gap at termination is of $8.12\%$ for \textbf{Meta-PBD}, against $34.20\%$ for \textbf{Single}. This large difference is essentially due to the quality of the primal solutions computed, as overall, \textbf{Meta-PBD} determines transportation plans $20\%$ more cost-effective than \textbf{Single}.

\section{Conclusions and future work}
\label{sec:conclusion}
\noindent
In this paper, we studied the Logistics Service Network Design Problem (LSNDP), a transportation planning problem that arises in the management of supply chains. We proposed a solution algorithm based on the recently-proposed \textit{Partial Benders Decomposition} (PBD) technique, wherein information derived from subproblem(s) is used to reinforce the master problem solved in the context of a Benders decomposition-based algorithm. In the proposed approach, the master problem is strengthened with variables and constraints that model the need to route one or more ``super-products'' that are aggregations of a subset of products. 

Existing implementations of PBD formulate and use a single master problem to solve the original problem. We proposed a solution framework called \textit{Meta Partial Benders Decomposition} (Meta-PBD) that changes the master problem considered throughout its execution based on information regarding the progress of the optimization process. Through an extensive computational study, we demonstrated that Meta-PBD strictly outperforms partial Benders decomposition-based algorithms based on a single master problem. More specifically, by changing the master problem used during its execution, Meta-PBD diversifies the search, allowing it to find primal solutions significantly better than those computed by the benchmarks and to solve more instances to optimality. 

There are several avenues for enhancing Meta-PBD, both in the context of solving the LSNDP and other network design problems. For example, while Meta-PBD currently starts with a master problem based on a single super-product, we are considering enhancing it to potentially start with a master problem based on multiple super-products. Also, the concept of aggregating products/commodities at different levels of granularity can be applied to all multi-product/commodity network design problems. Therefore, another avenue of future work is to develop a Meta-PBD-type scheme for solving general multi-product/commodity network design problems.



\section*{Acknowledgements}
This work has been partially supported by the French National Research Agency through the Pi-Comodalit\'e project under the grant ANR-15-CE22-0012. This support is gratefully acknowledged.

\bibliographystyle{apalike}
\bibliography{biblio}

\section{Appendix}
\label{sec:appendix}
\noindent
In this appendix we present proofs of results discussed in the main body of this paper as well as detailed pseudo-code for the proposed Meta Partial Benders Decomposition algorithm.
\begin{theorem}
The $\K$-enhanced master problem, \textbf{K-EMP}, is a relaxation of the Logistics Service Network Design problem, LSNDP.
\end{theorem}

\begin{proof}
We prove this claim by showing that any feasible solution for the LSNDP is also feasible for the \textbf{K-EMP} and has the same objective function value. To do so, we let $(\flow,\truck)$ be a feasible solution of the LSNDP. Let consider a solution $(\flow^{\superproducts}, \truck, \z)$ such that: 
$$ \flow^{\superproduct_k \time\time'}_{ij}=  \sum\limits_{\product \in \products_k} \flow^{\product\time\time'}_{ij},  \hspace{0.15cm} \forall ((i,\time),(j,\time')) \in \exptarcs \cup \expharcs,  \forall \superproduct_k \in \superproducts, \products^i \cap \products_k  \neq \emptyset$$ 
$$ \z = \sum\limits_{\tarc \in \exptarcs}   \sum\limits_{\superproduct_k \in \superproducts}  \ltcost \flow^{\superproduct_k \time\time'}_{ij}  +  \sum\limits_{\harc \in \expharcs}   \sum\limits_{\superproduct_k \in \superproducts}  \lhcost \flow^{\superproduct_k \time\time+1}_{ii} $$ 
It is easy to prove that this solution is feasible for the \textbf{K-EMP}. By construction, for each variable $\flow^{\product\time\time'}_{ij}$ in the LSNDP, $\product \in \products_k$, there is a corresponding variable $\flow^{\superproduct_k \time\time'}_{ij}$ in the \textbf{K-EMP}. We know that for each warehouse $(j, \time')$ and each product $\product \in \products$: $\sum\limits_{\tarc \in \exptarcs \cup \expharcs} \flow^{\product\time\time'}_{ij} - \sum\limits_{((j,\time'),(l,\time'')) \in \exptarcs \cup \expharcs} \flow^{\product\time'\time''}_{jl} = 0$. If we sum this expression on the products of $\products_k$, we obtain: $$\sum\limits_{\tarc \in \exptarcs \cup \expharcs}  \sum\limits_{\product \in \products_k} \flow^{\product\time\time'}_{ij} - \sum\limits_{((j,\time'),(l,\time'')) \in \exptarcs \cup \expharcs}  \sum\limits_{\product \in \products_k} \flow^{\product\time'\time''}_{jl} = 0 $$ $$ \Longleftrightarrow \sum\limits_{\tarc \in \exptarcs \cup \expharcs} \flow^{\superproduct_k \time\time'}_{ij} - \sum\limits_{((j,\time'),(l,\time'')) \in \exptarcs \cup \expharcs} \flow^{\superproduct_k \time'\time''}_{jl} = 0  $$
Therefore $(\flow^{\superproducts},\truck, \z)$ respects constraints \eqref{flowConservationConstraintKEMP}. As $(\flow,\truck)$ respects constraints \eqref{demandConstraintFULL}-\eqref{warehouseCapacityConstraintFULL}-\eqref{truckActualizationConstraintFULL}, it is trivial to demonstrate $(\flow^{\superproducts},\truck, \z)$ also respects constraints \eqref{demandConstraintKEMP}-\eqref{warehouseCapacityConstraintKEMP}-\eqref{truckActualizationConstraintKEMP}. By construction of $\z$, $(\flow^{\superproducts},\truck, \z)$ respects constraint \eqref{zConstraintKEMP} which makes it an admissible solution for the \textbf{K-EMP}. Let $Q(\flow,\truck)$ be the objective value of $(\flow,\truck)$:
$$ Q(\flow,\truck) \hspace{0.1cm}  = \hspace{0.1cm} \sum\limits_{\tarc \in \exptarcs}     \fcost \truck^{\time\time'}_{ij}  +   \sum\limits_{\tarc \in \exptarcs} \sum\limits_{\product \in \products}    \ltcost \flow^{\product\time\time'}_{ij} +   \sum\limits_{\harc \in \expharcs}   \sum\limits_{\product \in \products} \lhcost \flow^{\product\time\time+1}_{ii}  $$ 
$$ = \hspace{0.1cm} \sum\limits_{\tarc \in \exptarcs}     \fcost \truck^{\time\time'}_{ij}  +   \sum\limits_{\tarc \in \exptarcs} \sum\limits_{k=1}^{\K} \sum\limits_{\product \in \products_k}   \ltcost \flow^{\product\time\time'}_{ij} +   \sum\limits_{\harc \in \expharcs}  \sum\limits_{k=1}^{\K} \sum\limits_{\product \in \products_k} \lhcost \flow^{\product\time\time+1}_{ii} \hspace{0.1cm}  $$
$$ = \hspace{0.1cm}  \sum\limits_{\tarc \in \exptarcs}     \fcost \truck^{\time\time'}_{ij}   +   \sum\limits_{\tarc \in \exptarcs}   \sum\limits_{\superproduct_k \in \superproducts}  \ltcost \flow^{\superproduct_k \time\time'}_{ij}  +  \sum\limits_{\harc \in \expharcs}   \sum\limits_{\superproduct_k \in \superproducts}  \lhcost \flow^{\superproduct_k \time\time+1}_{ii} \hspace{0.1cm} $$ $$  = \hspace{0.1cm}  \sum\limits_{\tarc \in \exptarcs}     \fcost \truck^{\time\time'}_{ij} + \z \hspace{0.1cm}  = \hspace{0.1cm}  Q(\flow^{\superproducts},\truck, \z) $$
Solution $(\flow^{\superproducts_k}, \truck, \z)$ that replicates solution $(\flow,\truck)$ by an aggregation of flows, is feasible to the \textbf{K-EMP}. The two solutions have an identical objective function value. Thus, \textbf{K-EMP} is a relaxation of LSNDP.
\end{proof}

We next present the proof that in the special case that given a $\K$-partition of the products $\{ \products_1,...,\products_\K \}$, if for all subsets of $\products_k$, each supplier that provides any product $\product \in \products_k$ can also provide all other products included in $\products_k$, then all product subsets of the $\K$-partition can be aggregated to yield a \textbf{K-EMP} that  is equivalent to the original problem.
\begin{theorem}
Let $\{ \products_1,...,\products_\K  \}$ be a $\K$-partition of the products, and $\superproduct_k \in \superproducts, k \in \{1,...,\K\},$ the associated super-products. If for each $k \in \{1,...,\K\}$, each supplier $s \in \suppliers$ that provides any product $\product \in \products_k$ can also provide all other products included in $\products_k$, then the \textbf{K-EMP} is equivalent to the LSNDP.
\end{theorem}

\begin{proof}
We proved in subsection \ref{subsec:reinforce} that the \textbf{K-EMP} is a relaxation of the LSNDP. We demonstrate now that, in the case stated above, the LSNDP is a relaxation of the \textbf{K-EMP}. The premise of the proof is as follows. A solution of the \textbf{K-EMP} can be seen as a set of paths transporting super-products from suppliers to customers. More specifically, for each super-product $\superproduct_k$, paths from suppliers of $\superproduct_k$ to customer $(c,t)$ transit  $\D^{\superproduct_k}_{\customer\time}$ units of $\superproduct_k$. Thus, for each customer, the incoming flow of $\superproduct_k$ can be divided into $|\products_k|$ parts of quantity $\d^\product_{\customer\time}$ each, $p \in \products_K$. By converting each part into a flow of the corresponding product $p$, one can construct a solution for the LSNDP. Since each suppliers of $\superproduct_k$ provide all products of $\products_k$, the obtained solution is feasible for the LSNDP.

Let $(\flow^{\superproducts}, \truck, \z)$ be a feasible solution of the \textbf{K-EMP}. Let $\paths$ be the set of paths from suppliers to customers in $\expnetwork$. By definition the flow of super-products $\flow^\superproducts$ originates from suppliers, ends at customers, and respects flow conservation at each warehouse. Therefore we can decompose the flow of super-product $\flow^\superproducts$ into paths of $\Lambda$, and obtain an equivalent path-solution $\pathflow^\superproducts$. We denote $\pathflow_\path^{\superproduct_k}$ the quantity of $k^{th}$ super-product transported along a path $\path \in \paths$.

Let $\paths_{\customer\time}$ be the set of paths from suppliers to customer $(\customer, \time)$. As $(\pathflow^{\superproducts},\truck, \z)$ respects the \textbf{K-EMP} demand constraint, for each super-product $\superproduct_k$ the sum of flows along paths of $\paths_{\customer\time}$ sustains $(\customer, \time)$ demand of $\superproduct_k$, i.e. $\sum\limits_{\path \in \paths_{\customer\time}}  \pathflow_\path^{\superproduct_k}  = \D^{\superproduct_k}_{\customer\time} = \sum\limits_{\product \in \products_k}  \d^{\product}_{\customer\time}.$ Thus, for each a customer $(\customer, \time)$, we can partition its incoming path-flow of super-product $\superproduct_k$ into $|\products_k|$ parts $\{\dot \pathflow^1,...,\dot  \pathflow^{|\products_k|}  \}$ transiting $\d^\product_{\customer\time}$ units of $\superproduct_k$ each, $\product \in \products_k$. As a result, for each product $\product \in \products_k$ the total flow along paths in the $p^{th}$ subset of the flow-partition equals $\d^\product_{\customer\time}$, i.e. $\sum\limits_{\path \in \paths_{\customer\time}} \dot  \pathflow^{\product}_{\lambda} = \d^\product_{\customer\time}$. In addition, for each path $\path \in \paths_{\customer\time}$, the sum of flows over the flow-partition equals the total flow of super-product $\superproduct_k$, i.e. $\sum\limits_{\product \in \products_k} \dot  \pathflow^{\product}_{\lambda} = \pathflow^{\superproduct_k}_{\lambda}$.

Based on the flow-partition of $\pathflow^\superproducts$, we construct a path-solution $\pathflow$ for the LSNDP. We denote $\pathflow_\path^{\product}$ the quantity of product $\product$ transiting along path $\path$. For each customer $(\customer, \time)$, each $k \in \{1,...,\K\}$, and each product $\product \in \products_k$, if $ \dot  \pathflow^{\product}_{\lambda} > 0$ then the supplier of origin manufactures $\superproduct_k$ in the \textbf{K-EMP}. By hypothesis, the supplier of origin manufactures all products of $\products_k$ in the original problem. As a result, $\forall \product \in \products_k$ we can set any value for $\pathflow_{\path}^{\product}$. Thus, for each customer $(\customer, \time)$, each $k \in \{1,...,\K\}$, each path $\path \in \paths_{\customer\time}$, and each product $\product \in \products_k$, we set $\pathflow_{\path}^{\product}$ value to $ \dot \pathflow^{\product}_{\lambda}$.

Let $\flow$ be the flow-solution equivalent to path-solution $\pathflow$. We now demonstrate that $(\flow,\truck)$ is feasible for the LSNDP.  By definition, on a path, flow conservation is ensured. Thus, flow-solution $\flow$ respects constraints \eqref{flowConservationConstraintFULL}. For each customer $(\customer, \time) \in \customers_\timehorizon$, each $k \in \{1,...,\K\}$, and each product $\product \in \products_k$ we have $\sum\limits_{\path \in \paths_{\customer\time}} \pathflow^{\product}_{\lambda} = \d^\product_{\customer\time}$. By extension, the associated flow-solution $\flow$ respects demand constraints \eqref{demandConstraintFULL}. As $(\flow^\superproducts, \truck, \z)$ is a solution of the \textbf{K-EMP}, vehicle allocation $\truck$ is sufficient to route $\pathflow^\superproducts$. For each customer $(\customer, \time)$, each $k \in \{1,...,\K\}$, and each path $\path \in \paths_{\customer\time}$, $\sum\limits_{\product \in \products_k} \pathflow^{\product}_{\path} = \pathflow^{\superproduct_k}_{\path}$ therefore vehicle allocation $\truck$ is sufficient to route $\pathflow$, and so $\flow$. Thus, $(\flow, \truck)$ respects constraints \eqref{truckActualizationConstraintFULL}, and is a feasible solution to the LSNDP.

Let $(\flow^\superproducts, \truck, \z)$ objective value be: $$ Q(\flow^\superproducts, \truck, \z) \hspace{0.1cm}  = \hspace{0.1cm}  \sum\limits_{\tarc \in \exptarcs}   \fcost \trucksvariables + \z$$
Due to constraints \eqref{zConstraintKEMP} we have: 
$$ Q(\flow^\superproducts, \truck, \z) \hspace{0.1cm}  = \hspace{0.1cm}  \sum\limits_{\tarc \in \exptarcs}   \fcost \trucksvariables + \sum\limits_{\tarc \in \exptarcs} \sum\limits_{k=1}^{\K}  \ltcost \flowstkvariablesaggregated  +  \sum\limits_{\harc \in \expharcs}  \sum\limits_{k=1}^{\K}  \lhcost \flowshkvariablesaggregated $$
By partitioning, for each $k \in \{1,...,\K\}$, for each product $\product \in \products_k$ and for each path $\path \in  \paths_{\customer\time}$: $\sum\limits_{\product \in \products_k} \pathflow^{\product}_{\lambda} = \pathflow^{\superproduct_k}_{\lambda}$. Regarding flow-solutions $\flow$ and $\flow^\superproducts$ associated to $\pathflow$ and $\pathflow^\superproducts$, that means on each arc, the sum of flows over the products of $\products_k$ equals the flow of super-product $\superproduct_k$, i.e. $\sum\limits_{\product \in \products_k} \flow^{\product\time\time'}_{ij} = \flow^{\superproduct_k \time\time'}_{ij}.$ Therefore: $$ Q(\flow^\superproducts, \truck, \z) \hspace{0.1cm}  = \sum_{\tarc \in \exptarcs}     \fcost \trucksvariables  \hspace{0.2cm} +   \sum_{\tarc \in \exptarcs}  \sum\limits_{k=1}^{\K} \sum\limits_{\product \in \products_k} \ltcost \flowstvariables \hspace{0.2cm} + \sum_{\harc \in \expharcs} \sum\limits_{k=1}^{\K} \sum\limits_{\product \in \products_k}  \lhcost \flowshvariables $$

$$ = \hspace{0.1cm}  \sum_{\tarc \in \exptarcs}     \fcost \trucksvariables  \hspace{0.2cm} + \sum_{\tarc \in \exptarcs}  \sum_{\product \in \products}    \ltcost \flowstvariables \hspace{0.2cm} + \sum_{\harc \in \expharcs} \sum_{\product \in \products}   \lhcost \flowshvariables   = Q(\flow, \truck) $$

Solution $(\flow,\truck)$ that replicates solution $(\flow^\superproducts, \truck, \z)$ by partitioning of flows, is feasible to the LSNDP. The two solutions have an identical objective function value. Thus, the LSNDP is a relaxation of the \textbf{K-EMP}. As the \textbf{K-EMP} is also a relaxation of the LSNDP, we demonstrated that the \textbf{K-EMP} and the LSNDP are equivalent if for each $k \in \{1,...,\K\}$, each supplier $s \in \suppliers$ that provides any product $\product \in \products_k$ can also provide all other products included in $\products_k$.
\end{proof}

We next present the proof that given a \textbf{K-EMP} for a given $\K$, one can always create a \textbf{K-EMP} based on $\K$ super-products that is at least as strong. 

\begin{theorem}
Given a $\K$-enhanced master problem, $\K \in \{1,...,|\products|-1\}$, there always exist a $\K+1$-enhanced master problem such that \textbf{K-EMP} is a relaxation of \textbf{K+1-EMP}
\end{theorem}

\begin{proof}
We demonstrate that for any \textbf{K-EMP} there exists a \textbf{K+1-EMP} such that any feasible solution for the \textbf{K+1-EMP} is also feasible for the \textbf{K-EMP}, and has the same objective function value. For this, we refine the $\K$-partition associated to the \textbf{K-EMP} by dividing one of its subsets into two parts.

Let consider a \textbf{K-EMP} based on a $\K$-partition of the product set: $\{ \products_1,...,\products_\K \}$. Subsets are sorted in ascending order with respect to their number of elements; thus the last subset $\products_\K$ contains at least two products. Let $\{\products^{'}_1,...,\products^{'}_{\K+1} \}$ be a $\K+1$-partition of the product set, such that $\products_i = \products^{'}_i,$ $\forall i \in \{1,...,\K-1\}$, and $\products^{'}_{\K} \cup \products^{'}_{\K+1} = \products_{\K},$ $\products^{'}_{\K} \neq \emptyset,$ $\products^{'}_{\K+1} \neq \emptyset$. 

In the resulting \textbf{K+1-EMP}, super-product $\superproduct^{'}_k$ is equivalent to super-product $\superproduct_k$ of the \textbf{K-EMP}, $\forall k \in \{1,...,\K-1 \}.$ Thus, for each $k \in \{1,...,\K-1 \}$ and for each variable $\flow^{\superproduct^{'}_k\time\time'}_{ij}$ in the \textbf{K+1-EMP}, there is a corresponding variable $\flow^{\superproduct_k \time\time'}_{ij}$ in the \textbf{K-EMP}. In addition, as super-products $\superproduct^{'}_{\K}$ and $\superproduct^{'}_{\K+1}$ include products in $\products_{\K}$, for each variable $\flow^{\superproduct^{'}_{\K}\time\time'}_{ij}$ and each variable $\flow^{\superproduct^{'}_{\K+1}\time\time'}_{ij}$ in the \textbf{K+1-EMP}, there is a corresponding variable $\flow^{\superproduct_{\K} \time\time'}_{ij}$ in the \textbf{K-EMP}.

Let $(\flow^{\superproducts^{'}}, \truck, \z)$ be a feasible solution for the \textbf{K+1-EMP}. Let consider a solution $(\flow^{\superproducts}, \truck, \z)$ for the \textbf{K-EMP}, such that the flow value of super-product $\superproduct_k$ is similar to the flow value of super-product $\superproduct^{'}_k$ in the \textbf{K+1-EMP}, $\forall k \in \{1,...,\K-1 \},$ and the flow value of super-product $\superproduct_{\K}$ aggregates the flow values of super-products $\superproduct^{'}_{\K}$ and $\superproduct^{'}_{\K+1}$ in the \textbf{K+1-EMP}. By using a similar reasonning to that in the proof of Theorem \ref{theorem1}, it is trivial to demonstrate that solution $(\flow^{\superproducts}, \truck, \z)$ is feasible for the \textbf{K-EMP}, and has the same objective function value that $(\flow^{\superproducts^{'}}, \truck, \z)$. Thus, \textbf{K-EMP} is a relaxation of \textbf{K+1-EMP}.
\end{proof}

\begin{algorithm}
  \small
  \DontPrintSemicolon
  \KwData{Maximum number of super-products, $\K_{max}$, Time limit on bounds improvement, $\timelimitlb$, Minimal improvement required for the bounds, $\imprlb\%$, Threshold on the number of master solutions explored, $\minmastersols$, $\K$-partitions of the product set for $\K$ ranging from $1$ to $\K_{max}$, Time limit, $\timelimitphaseone$}
  $\K=1$ \;
  $K^- = \emptyset$ and $K^+ = \emptyset$\;
  $LB = 0$ and $UB = \infty$\;
    \While{Time left AND No optimal solution found}{
        Formulate the \textbf{K-EMP}\;
        Initiate it with all Benders cuts found previously, LB, and UB\;
        \Begin(Solve \textbf{K-EMP} using the algorithm proposed by Belieres et al. \cite{Belieres:2019}:){
            \If{No time left OR optimal solution found}
               {Break}
            \ElseIf{UB/LB did not improved more than $\imprlb$ in the last $\timelimitlb$ seconds}{
                \If{Less than $\minmastersols$ produced in the last $\timelimitlb$ seconds}{
                    \If{$\K$ can increase}{
                        $\K = \lceil \frac{\K - K^-}{2} \rceil $ \;
                        Break\;
                    }
                }
                \Else{
                \If{$\K$ can decrease}{
                    \If{$K^+ = \emptyset$}{
                        $\K = \lceil \frac{\K_{max} - \K}{2} \rceil$\;
                        Break\;
                     }
                     \Else{
                        $\K = \lceil \frac{K^+ - \K}{2} \rceil $\;
                        Break\;
                     }
                    }
                }
            }
        } 
    Update $K^-$ and $K^+$\;
    Keep all Benders cuts in memory\;
    Update best lower bound LB and best upper bound UB\;
}
  \KwResult{Best lower bound, LB, and best upper bound, UB}
  \caption{\label{alg:phase1}%
    Phase I}
\end{algorithm}

\begin{algorithm}
    \small
    \DontPrintSemicolon
    \KwData{Initial lower bound, LB, Initial upper bound, UB, Time limit $\timelimitphasetwo$}
    1-partition $\leftarrow$ whole product set $\products$ \;
    \For{$\K$ ranging from $1$ to $|\products|-1$}{
        Identify product subset from the $\K$-partition with a matching rate lower than 1\;
        \If{No product subset identified}{
            Break\;
        }
        \Else{
            Partition it by applying the 2-medoids algorithm, using distance $1-m( p, p')$ for each $( p, p') \in \products$\;
            $\K+1$-partition $\leftarrow$ combine the obtained subsets with all other product subsets from the $\K$-partition unchanged \;
        }
    }
    Formulate the \textbf{K-EMP} associated to the last partition found\;
    Initiate it with LB and UB\;
    Solve \textbf{K-EMP} with a generic branch-and-cut implementation using a time limit $\timelimitphasetwo$ \;
    \KwResult{Final solution, UB}
    \caption{\label{alg:phase2}%
    Phase II}
\end{algorithm}

\end{document}